\documentclass{article}

\PassOptionsToPackage{numbers, compress}{natbib}

    \usepackage[preprint]{neurips_2025}

\usepackage[utf8]{inputenc} 
\usepackage[T1]{fontenc}    
\usepackage{hyperref}       
\usepackage{url}            
\usepackage{booktabs}       
\usepackage{amsfonts}       
\usepackage{nicefrac}       
\usepackage{microtype}      
\usepackage[dvipsnames]{xcolor}         
\usepackage{amssymb}
\usepackage{amsthm}
\usepackage{algorithm}
\usepackage{amsmath} 
\usepackage{algorithmicx}
\usepackage{algpseudocode}
\newtheorem{theorem}{Theorem}
\newtheorem{lemma}{Lemma}
\newtheorem{corollary}{Corollary}
\usepackage{graphicx}
\usepackage{subcaption}
\usepackage{wrapfig}
\usepackage{lipsum} 
\usepackage{wrapfig} 
\usepackage{pifont}
\usepackage{comment}
\definecolor{mintgreen1}{HTML}{98FF98}

\colorlet{mintgreen}{mintgreen1!15!white}

\usepackage{pdfcomment}
\usepackage{todonotes}
\usepackage[commandnameprefix=always]{changes}
\usepackage{soul}
\usepackage{bm}
\usepackage{authblk}
\renewcommand{\bm}[1]{\mathbf{#1}}
\title{FRAM: Frobenius-Regularized Assignment Matching with Mixed-Precision Computing}

\author[1,2]{Binrui Shen}
\author[3]{Yuan Liang}
\author[4,5]{Shengxin Zhu$^\nmid$}

\affil[1]{School of Mathematical Sciences, Laboratory of Mathematics and Complex Systems, MOE, 
Beijing Normal University, Beijing 100875, P.R. China}
\affil[2]{Faculty of Arts and Sciences, Beijing Normal University, Zhuhai 519087, P.R. China}
\affil[3]{School of Mathematical Sciences,
Beijing Normal University, Beijing 100875, P.R. China}

\affil[4]{Research Centers for Mathematics, Advanced Institute of Natural Science, Beijing Normal University, Zhuhai 519087, P.R. China}
\affil[5]{Guangdong Provincial Key Laboratory of Interdisciplinary Research and Application for Data Science, BNU-HKBU United International College, Zhuhai 519087, P.R. China}
\affil[ ]{\texttt{binrui.shen@bnu.edu.cn, l.y@mail.bnu.edu.cn, Shengxin.Zhu@bnu.edu.cn}}

\begin{document}

\maketitle

\begin{abstract}
Graph matching, typically formulated as a Quadratic Assignment Problem (QAP), seeks to establish node correspondences between two graphs. To address the NP-hardness of QAP, some existing methods adopt projection-based relaxations that embed the problem into the convex hull of the discrete domain.  However, these relaxations inevitably enlarge the feasible set, introducing two sources of error: numerical scale sensitivity and geometric misalignment between the relaxed and original domains. To alleviate these errors, we propose a novel relaxation framework by reformulating the projection step as a Frobenius-regularized Linear Assignment (FRA) problem, where a tunable regularization term mitigates feasible region inflation. This formulation enables normalization-based operations to preserve numerical scale invariance without compromising accuracy.  To efficiently solve FRA, we propose the Scaling Doubly Stochastic Normalization (SDSN) algorithm. Building on its favorable computational properties, we develop a theoretically grounded mixed-precision architecture to achieve substantial acceleration.
Comprehensive CPU-based benchmarks demonstrate that FRAM consistently outperforms all baseline methods under identical precision settings. When combined with a GPU-based mixed-precision architecture, FRAM achieves up to 370× speedup over its CPU-FP64 counterpart, with negligible loss in solution accuracy.
\end{abstract}

\section{Introduction}
\label{sec:intro}

Graph matching aims to find correspondences between graphs with potential relationships. It can be used in various fields of intelligent information processing, e.g., detection of similar pictures \cite{shen2020fabricated}, graph similarity computation \cite{lan2022aednet,lan2022more}, knowledge graph alignment \cite{xu2019cross}, autonomous driving \cite{song2023graphalign}, alignment of vision-language models \cite{nguyen2024logramedlongcontextmultigraph}, point cloud registration \cite{fu2021robust}, deep neural network fusion \cite{liu22k}, multiple object tracking \cite{he2021learnable}, and COVID-19 disease mechanism study \cite{gordon2020comparative}. However, it is known that the exact graph matching problem is a typical NP hard discrete optimization problem \cite{sahni1976p}, and it is computationally prohibitive to obtain a matching for large-scale graphs. 

To scale up the graph matching problem, many relaxation methods have been developed \cite{gold1996graduated, leordeanu2005spectral, cour2006balanced, 2010Reweighted, lu2016fast, shen2024adaptive}.  Such methods relax the problem to a continuous domain and then project the continuous solution back to the original discrete space. The doubly stochastic projection is a representative and frequently used recently \cite{zass2006doubly,lu2016fast}. The method maps the gradient matrix onto the convex hull of the original domain—is particularly representative. However, such relaxations inevitably enlarge the feasible region, leading to two key sources of errors: (1) geometric misalignment between the relaxed and original domains, which undermines the quality of the recovered integer solution; and (2) the numerical scale invariance of the quadratic objective, which is destroyed under the projections.  Such limitations motivate our current research.

In this work, we propose a novel relaxation framework within which the solution trajectory stays close to the original region, thereby enhancing solution accuracy and preserving numerical scale invariance. Furthermore, we introduce an innovative and theoretically grounded mixed-precision architecture, which leads to substantial acceleration of our graph matching algorithm.

Our main contributions are summarized as follows.

\begin{enumerate}
\item \textbf{Theoretical Results}:
We reformulate the doubly stochastic projection to a Frobenius-regularized linear assignment problem (FRA), where a tunable regularizer reduces relaxation-induced distortion. We further analyze the regularization parameter’s impact on performance and characterize the convergence point.
\item \textbf{Algorithm}:
We propose a graph matching algorithm FRAM that solves the QAP approximately by iteratively solving a sequence of FRA. To solve the FRA, we propose a Scaling Doubly Stochastic Normalization (SDSN), incorporating a parametric scaling mechanism that remains robust under varying numerical scales.
\item \textbf{Computing acceleration}: 
 We design a theoretically grounded mixed-precision architecture for the FRAM.  Compared to CPU-based double-precision computation, it achieves over 370× speedup on an NVIDIA RTX 4080 SUPER GPU in some tasks, without compromising accuracy. To the best of our knowledge, this is the first graph matching algorithm built upon a theoretically grounded mixed-precision architecture.
\end{enumerate}
\section{Related Works}
We briefly review three representative classes of algorithms relevant to our study: doubly stochastic optimization methods, spectral-based approaches, and optimal transport-based techniques. For a comprehensive overview of graph matching algorithms, readers are referred to existing survey articles such as \cite{emmert2016fifty, yan2016short}. Graduated assignment (GA) \cite{gold1996graduated} is one of the earliest continuous optimization algorithms. It transforms the quadratic assignment problem into a sequence of linear assignment approximations (called softassign). \citet{shen2024adaptive} develop an adaptive softassign that automatically tunes an entropic parameter. \citet{tian2012convergence} demonstrate that GA with discrete projections may get into a cycle solution. The integer projected fixed point (IPFP) \cite{leordeanu2009integer} projects a gradient matrix into a permutation matrix and updates the solution by a convex combination of the previous iterate and the projected gradient matrix. \citet{zass2006doubly} develop a doubly stochastic normalization (DSN) to find the nearest doubly stochastic matrix for a given symmetric matrix. \citet{lu2016fast} adapt the DSN \cite{zass2006doubly} so that the projected gradient matrix is doubly stochastic. Similar iterative formulas include the spectral-based algorithms \cite{leordeanu2005spectral, shen2025lightning} that recover assignments by transforming the leading eigenvector of a compatibility matrix into the matching matrix. \citet{cour2006balanced} further add affine constraints to better approximate the original problem, which retains spectral methods' speed and scalability benefits. \citet{hermanns2023grasp} establish a correspondence between functions obtained from eigenvectors of the Laplacian matrix, which encode multiscale structural features. For OT-based algorithms, Gromov-Wasserstein Learning (GWL) \cite{xu2019gromov} measures the distance between two graphs by the Gromov-Wasserstein discrepancy and matches graphs by optimal transport. S-GWL \cite{xu2019scalable}, a scalable variant of GWL, divides matching graphs into small graphs for matching to enhance efficiency. 

Mixed-precision computing is a sophisticated technology for accelerating computationally intensive applications, Such technology has been successfully applied for linear systems \cite{connolly2021stochastic} \cite{henry2019leveraging}and Deepseek-V3 \cite{liu2024deepseek} while little attention is received in the graph matching context due to a shortage of theoretical insights. The principal challenges in mixed-precision computing stem from (1) the limited dynamic range of lower-precision formats, which induces a risk of overflow, and (2) cumulative numerical errors during lower-precision operations, necessitating rigorous mathematical analysis for error propagation characterization and stability guarantees \cite{narang2017mixed}. See \cite{higham2022mixed} for details on the mixed-precision algorithm. In this paper, we obtained some theoretical insights which can guarantee a mixed precision algorithm works stably in the graph matching context and achieves a significant speedup. 


\section{Preliminaries}
This section introduces the attributed graph, the matching matrix for graph correspondences, the formulation of graph matching problems, and a projected fixed-point method for solving them.

\textbf{Graph}. An \textit{undirected attributed graph} $G=\{V,E,A,F\}$ consists of a finite set of nodes $V = \{1, \dots ,n\}$ and a set of edges $E \subset V \times V$. $A$ is a nonnegative symmetric \textit{edge attribute matrix} whose element $A_{ij}$ represents an attribute of $E_{ij}$. The $i_{th}$ row of feature matrix $F$ represents the attribute vector of $V_{i}$. 

\textbf{Matching matrix}. Given two attributed graphs $G=\{V,E,A,F\}$ and $\tilde{G}=\{\tilde{V},\tilde{E},\tilde{A},\tilde{F}\}$, we first consider the same cardinality of the vertices $n = \tilde{n}$ for simplicity. A matching matrix $M\in \mathbb{R}^{n\times n}$ can encode the correspondence between nodes:
\( M_{i\tilde{i}} = 1 \) if node \( i \) in \( G \) matches node \( \tilde{i} \) in \( \widetilde{G} \), and \( M_{i\tilde{i}} = 0 \) otherwise. Subject to the one-to-one constraint, a matching matrix is a permutation matrix. The set of permutation matrices is denoted as $\Pi_{n \times n} = \{M: M\mathbf{1}=\mathbf{1},{M}^{T}\mathbf{1}=\mathbf{1}, {M}\in\{0,1\}^{n\times n}\}$, where $\mathbf{1}$ represents vectors with all-one.

\textbf{Problem and continuous optimization}. The graph matching problem is normally formulated as a quadratic assignment problem (QAP) that is NP-hard \cite{garey1979computers}.  A common trick for solving such discrete problems is relaxation that first finds a solution on $\mathcal{D}_{n \times n}:=\{N: N\mathbf{1}=\mathbf{1}, N^{T}\mathbf{1}=\mathbf{1}, N \geq 0$\} which is the convex hull of the original domain:
\begin{equation}
\begin{array}{cc}
   N^* =  \arg\max\limits_{N\in \mathcal{D}_{n \times n}}    \Phi(N), \quad  \Phi(N)=\frac{1}{2} \underbrace{ \langle {N}^{T} {A} {N}, {\widetilde{A}}\rangle}_{\text{Edges' similarites}}+ \lambda \underbrace{\langle{N}^{T}{F,\tilde{F}}\rangle}_{\text{Nodes' similarites}},  
\end{array}
\label{eq.Object}
\end{equation}
where $\lambda$ is a parameter and $\langle \cdot,\cdot\rangle$ represents the Frobenius inner product. $N^*$ is transformed back to the original discrete domain $\Pi_{n\times n}$ by solving a \textit{linear assignment problem}:
\begin{equation}
     M = \arg \min_{P\in \Pi_{n \times n}} \| P -  {N^*}\|_{F}.
    \label{eq. linear assignment}
\end{equation}
The matrix $M$ is the final solution for graph matching.

\textbf{The Projected Fixed-Point Method}.  
Many methods \cite{gold1996graduated,cour2006balanced, leordeanu2009integer, lu2016fast,shen2024adaptive}  adopt the same iterative framework to efficiently solve \eqref{eq.Object}:
\begin{equation}
\begin{array}{cc}
    N^{(t+1)}=(1- \alpha)N^{(t)}  + \alpha D^{(t)} ,
    \label{iter.DSPFP}\\
    D^{(t)}=\mathcal{P}(\nabla \Phi(N^{(t)}))  = \mathcal{P}(AN^{(t)}\tilde{A}+\lambda  F\tilde{F}^T),
\end{array}
\end{equation}
where $\alpha$ is a step size parameter and $\mathcal{P}(\cdot)$ is an operator to enforce the gradient matrix into a certain region. When the solution domain is relaxed to the convex hull of the original domain (the set of doubly stochastic matrices), a natural choice of $\mathcal{P}(\cdot)$ is the doubly stochastic projection \cite{zass2006doubly,lu2016fast}. It finds the closest doubly stochastic matrix to the gradient matrix $\nabla \Phi(N^{(t)})$ in terms of the Frobenius norm:
\begin{equation}
\mathcal{P_{D}}(X) = \arg \min_{D \in \mathcal{D}_{n \times n}} \|D - X\|_{F}.
\label{eq:doubly_assign}
\end{equation}
The resulting algorithm is doubly stochastic projected fixed-point method (DSPFP) \cite{lu2016fast}.

\section{Projection to Assignment}
We first propose a regularized linear assignment formulation by exploring the numerical sensitivity of the doubly stochastic projection. We then analyze how the regularization parameter affects performance.
\label{sec:method}

\subsection{Doubly stochastic projection to assignment}
It is straightforward to observe that the solution to the quadratic assignment problem~\eqref{eq.Object} also maximizes \( w \Phi(N) \), where \( w \) is a positive scale constant. This reflects the numerical scale-invariant property of the objective. However, the doubly stochastic projection \(\mathcal{P}_{\mathcal{D}}(\cdot)\) fails to preserve this property:
\begin{equation}
\mathcal{P_{D}}(X) \neq \mathcal{P_{D}}(wX).
\end{equation}
 To explore the reason for this, consider:
\begin{equation}
         \mathcal{P_{D}}(wX) = \arg\min_{D \in \mathcal{D}_{n \times n}} \|D - wX\|_{F}  = \arg\min_{D \in \mathcal{D}_{n \times n}} \|D - wX\|^2_F.     
\end{equation}
By expanding the norm, we have
\begin{equation}
    \|D - wX\|^2_F =\langle D, D \rangle - 2\langle D, wX \rangle + \langle wX, wX \rangle.
\end{equation}

Since \(\langle wX, wX \rangle\) is invariant with respect to \(D\), it does not affect the optimization. Thus, 
\begin{equation}
\mathcal{P_{D}}(wX) = \arg\min_{D \in \mathcal{D}_{n \times n}} \langle D, D \rangle - 2w\langle D, X \rangle.
\end{equation}
As a result, the problem reduces to
\begin{equation}
\mathcal{P_{D}}(wX) = \arg\max_{D \in \mathcal{D}_{n \times n}} \langle D, X \rangle -\frac{1}{2w}\langle D, D \rangle.
\label{eq.fra}
\end{equation}
$\langle D, X \rangle$ represents an assignment score. Since \( X \) corresponds to the gradient matrix \( \nabla \Phi(N^{(t)}) \) in~\eqref{iter.DSPFP}, it is intuitive that higher assignment scores in~\eqref{eq.fra} tend to result in higher objective values \eqref{eq.Object} during the iterative refinement process.
As \( w \) increases, the projection process emphasizes optimizing the assignment score. Conversely, when \( w \) decreases, the significance of the assignment score diminishes. Therefore, the matching performance is sensitive to the scale constant \( w \).

To eliminate the impact of scaling variations, we normalize the input matrix \( X \) by replacing \( X \) with \( \frac{X}{\max(X)} \). Furthermore, we introduce a parameter \( \theta \) to explicitly control the significance of the assignment score. The new problem is formalized as follows.

\begin{theorem}
The solution of the doubly stochastic projection with a scaling parameter $\theta$ 
\begin{equation}
D_X^\theta = \arg\min_{D \in \mathcal{D}_{n \times n}} \|D - \frac{\theta}{2}X\|_{F}^2
\label{eq:dsp_assignment}
\end{equation}
is the solution of a Frobenius-regularized linear assignment problem
\begin{equation}
         D_X^\theta = \arg\max_{D \in \mathcal{D}_{n \times n}}  \Gamma^{\theta}(X),  \quad   \Gamma^{\theta}(X) = \langle D, X \rangle -\frac{1}{\theta}\langle D, D \rangle.
\end{equation}
\end{theorem}



\subsection{Parameter Impact}
To establish this analysis, we quantify the solution quality through a distance metric between $D^{\theta}_X$ and $D^{\infty}_X$ for a given matrix $X \in \mathbb{R}^{n \times n}$. 
The total assignment score can be defined as $\langle D^{\theta}, X \rangle$, which scales with the problem size. To enable scale-independent error analysis, we use a  \textit{normalized assignment error} to define the performance gap: 
\begin{equation} 
\epsilon^{\theta}_{X} = \frac{1}{n} \left( \langle D^{\infty}_X, X \rangle - \langle D^{\theta}_X, X \rangle \right) 
\end{equation} 
This normalization effectively decouples the approximation error from the problem scale, providing a stable metric to assess the quality of $D^{\theta}_X$ across varying nodes' numbers.


\begin{theorem}
For a nonnegative matrix $X \in \mathbb{R}^{n \times n}$, the following inequality holds:

\begin{equation}
\begin{array}{cc}
    \epsilon^{\theta}_{X}\leq  \frac{1}{\theta}. 
\end{array}
\end{equation}
\end{theorem}
\noindent This theorem shows that the performance gap between $D^\theta_X$ and $D^\infty_X$ is bounded by $1/\theta$. 

\subsection{Convergence to optimal assignment}
To better understand the behavior of FRA under varying \( \theta \), we analyze the limiting behavior of \( D_X^\theta \) as \( \theta \) approaches 0 and infinity. 
\begin{theorem}
    As $\theta \to \infty$, the matrix $D_X^\theta$ converges to the unique matrix $D^*$. \( D^* \) minimizes the regularization term within the set \( \mathcal{F} \), the convex hull of the optimal permutation matrices.
\end{theorem}
This theorem reveals a key advantage of FRA: unlike standard linear assignment solvers \cite{leordeanu2009integer} that return a single solution and discard others, \( D_X^\theta \) approximates a convex combination of all optimal permutations, preserving richer solution information.

\begin{corollary}
If there is only one optimal permutation, then $D^{\theta}_X$ converges to the corresponding permutation matrix.
\end{corollary}

\begin{theorem}
    As $\theta \to 0$, the matrix $D_X^\theta$ converges to the matrix $\frac{\mathbf{1}\mathbf{1}^T}{n}$.
\end{theorem}
Figure~\ref{fig:theta_compare} visualizes how the matrix \( D_X^\theta \) changes with different values of \( \theta \). When \( \theta \) is small, the matrix entries are nearly uniform; as \( \theta \) increases, the matrix progressively approaches a permutation matrix that lies within the original feasible domain of the QAP. This illustrates how \( \theta \) suppresses the bias introduced by relaxation. By selecting a proper value of \( \theta \), the intermediate solution in the matching process is constrained within a relaxed region that remains close to the original feasible domain of graph matching problems.

\begin{figure}[htbp]
    \centering
    \begin{subfigure}[b]{0.25\textwidth}
        \includegraphics[width=\textwidth]{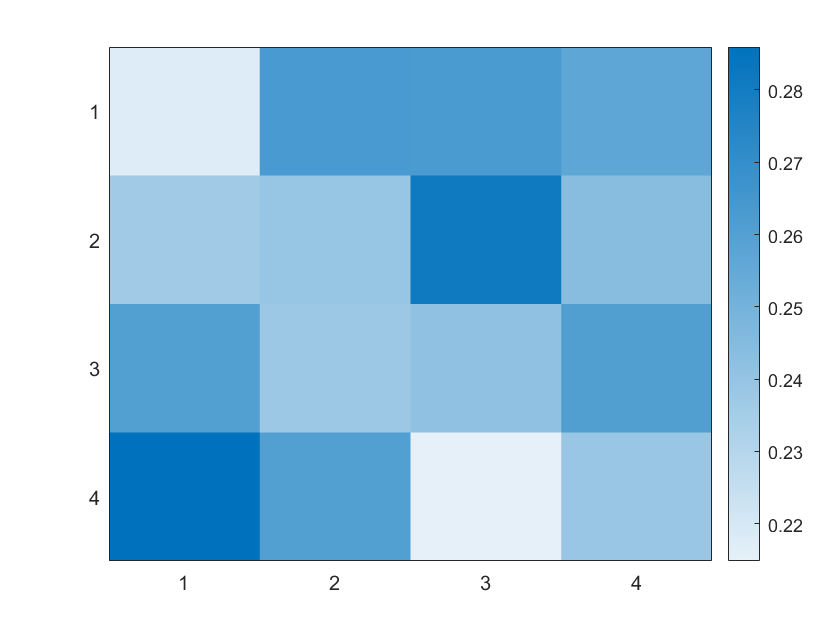}
        \caption{$\theta = 0.1$}
    \end{subfigure}
    \hfill
    \begin{subfigure}[b]{0.25\textwidth}
        \includegraphics[width=\textwidth]{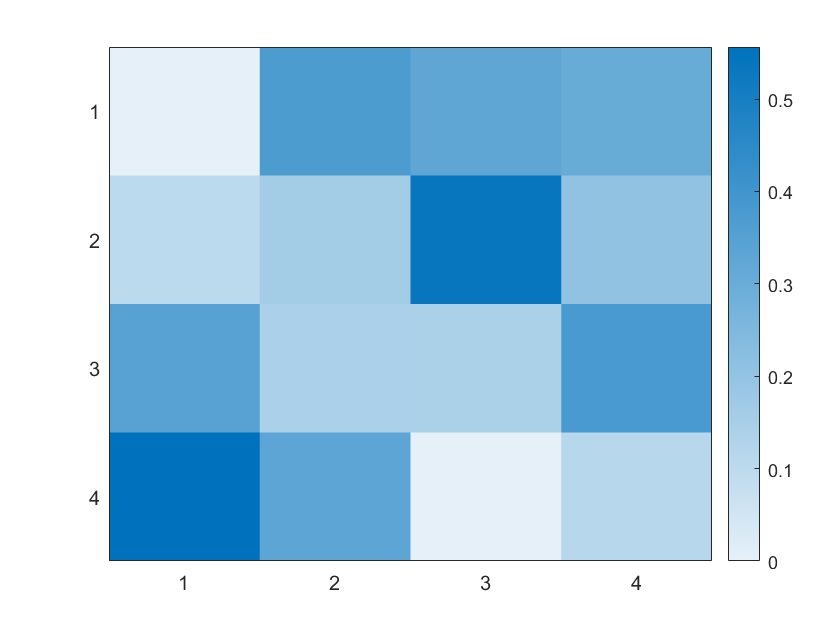}
        \caption{$\theta = 1$}
    \end{subfigure}
    \hfill
    \begin{subfigure}[b]{0.25\textwidth}
        \includegraphics[width=\textwidth]{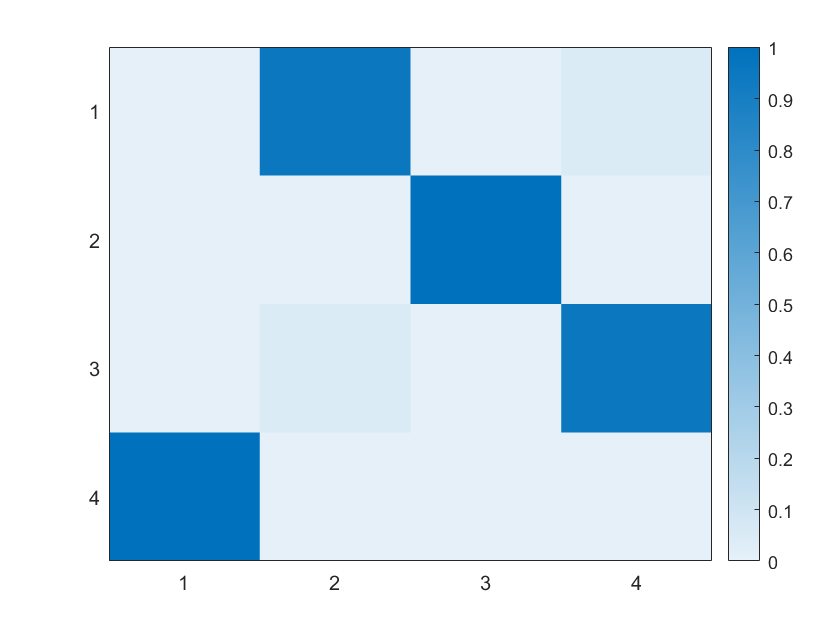}
        \caption{$\theta = 10$}
    \end{subfigure}
    \caption{Visualization of $D_X^\theta$ under different $\theta$ values. Cell color indicates the matrix value; darker colors correspond to larger entries.}
    \label{fig:theta_compare}
\end{figure}
When FRA serves as a module in graph matching \eqref{iter.DSPFP}, increasing $\theta$ typically improves the step-wise score by enforcing sharper matchings. However, excessively large $\theta$ may slightly degrade the final matching performance. This can be understood from a probabilistic perspective: each entry in the doubly stochastic matrix $D_X^\theta$ represents the probability of a potential match. A large $\theta$ leads to over-confident assignments too early, reducing flexibility and hindering the exploration of alternative assignments.

\section{Scaling Doubly Stochastic Normalization}
This section introduces a Scaling Doubly Stochastic Normalization (SDSN) to solve the FRA by demonstrating its applicability to asymmetric problems, defining a convergence criterion, and robustness to truncation residual. 

\subsection{Doubly Stochastic Normalization}
Owing to the equivalence between FRA and the scaling doubly stochastic projection, FRA admits a solution via tailored modifications to standard projection algorithms. \citet{zass2006doubly} solve the doubly stochastic projection \eqref{eq:doubly_assign} by alternating iteration between two sub-problems until convergence:
	\begin{equation}\mathcal{P}_1(X) = \arg\min_{Y\mathbf{1}=Y^T\mathbf{1} = \bm1} \| X - Y \|_{F},\quad \mathcal{P}_2(X) = \arg\min_{Y \geq 0} \| X - Y \|_{F}^2\ 
    \label{eq.p1}
    \end{equation}
The von-Neumann successive projection lemma \cite{vonNeumann1932FunctionalOperatorsII} states that $\mathcal{P}_2\mathcal{P}_1\mathcal{P}_2\mathcal{P}_1 \dots \mathcal{P}_2\mathcal{P}_1(X)$ will converge onto the $\mathcal{P_{D}}(X)$. The derived doubly stochastic normalization (DSN) \cite{zass2006doubly} for symmetric $X$ works as follows.
\begin{equation}
\tilde{X}^{(k)} =\mathcal{P}_1 \left( X^{(k-1)} \right), \ \  X^{(k)} =\mathcal{P}_2 \left(\tilde{X}^{(k)}  \right),
\label{eq.dsp_iter_simple}
\end{equation}
\begin{equation}
\mathcal{P}_1(X) =X + \left( \frac{{I}}{n} + \frac{\mathbf{1}^T X \mathbf{1}}{n^2}I - \frac{\mathbf{X}}{n} \right) \mathbf{1}\mathbf{1}^T - \frac{\mathbf{\mathbf{1} \mathbf{1}^T X}}{n} , \quad \quad \mathcal{P}_2(X) = \frac{X + |X|}{2},
\label{eq.P1_sol}
\end{equation}

where \( I \) is the \( n \times n \) identity matrix. It alternately applies row and column normalization and non-negativity enforcement to ensure that the resulting matrix satisfies the doubly stochastic property. Each iteration requires \( O(n^2) \) operations. 
	
\citet{zass2006doubly} design the DSN for symmetric matrices. Based on this, \citet{lu2016fast} further uses this method for asymmetric matrices. We provide the theoretical foundation as follows.
\begin{theorem}\label{thm1}
    For an asymmetric square matrix $X$, \eqref{eq.P1_sol} is the solution of the projection \eqref{eq.p1}.
\end{theorem}\label{The.1}

\subsection{Convergence criterion}
An appropriate convergence criterion is notably absent in DSN. \citet{zass2006doubly} terminate iterations when the updated matrix is doubly stochastic. However, this approach can be computationally expensive, leading to inefficiency in large-scale tasks. In contrast, \citet{lu2016fast} fix the number of iterations at 30; while this strategy improves speed, it does not guarantee that the output matrix is doubly stochastic. 

\begin{wrapfigure}{r}{0.42\textwidth}
    \centering
    \vspace{-8pt}  
    \includegraphics[width=\linewidth]{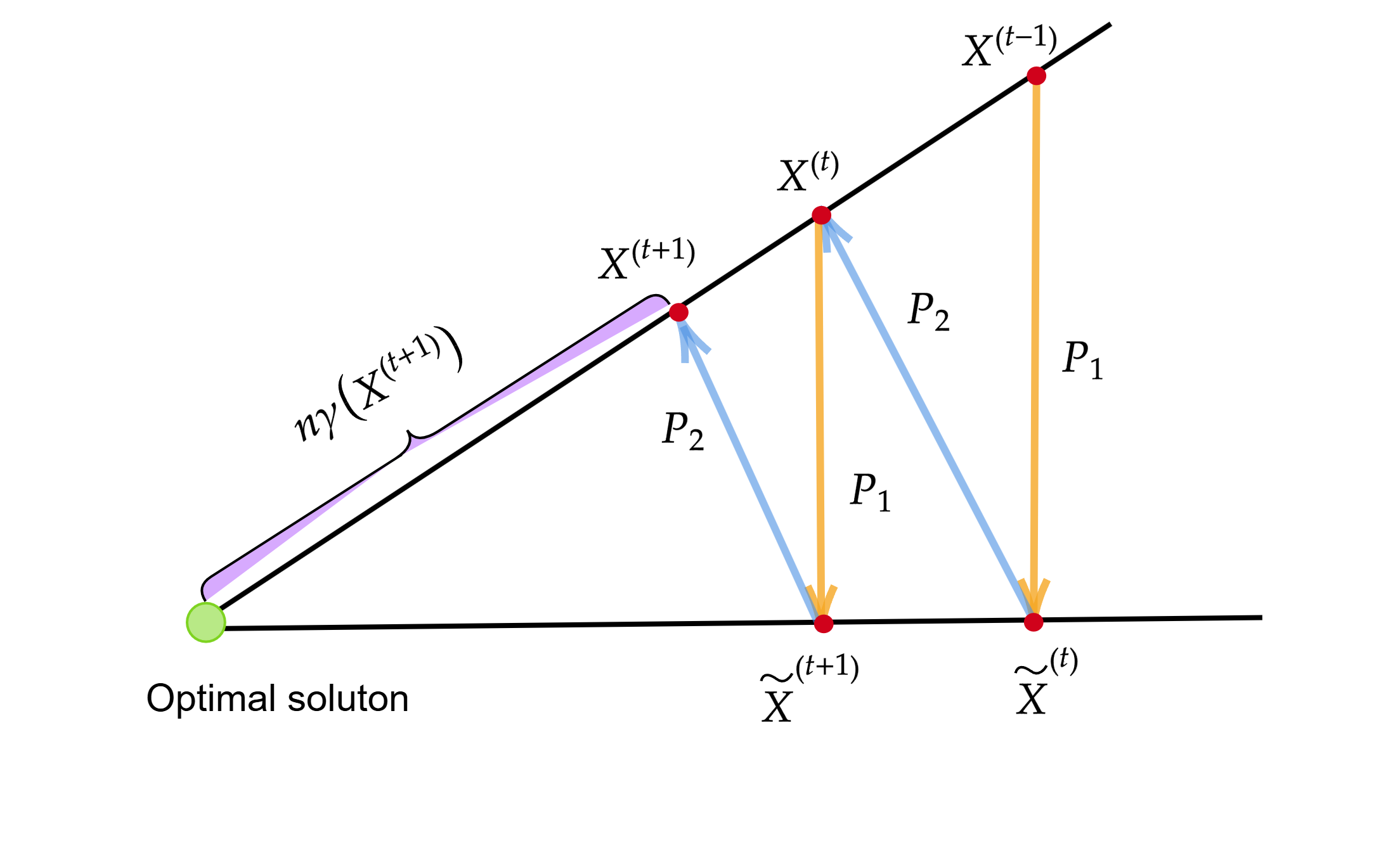}
    \caption{Convergence process of SDSN.}
    \label{Fig.convergence}
\end{wrapfigure}

We propose a criterion to quantitatively measure the distance between the current matrix and the ideal solution, ensuring theoretical soundness, dimension independence, and computational efficiency. Motivated by this, we define a \textbf{dimension-invariant distance measure} as
\begin{equation}
    \gamma(X^{(k)}) = \frac{1}{n}\sum_{i,j} X^{(k)}_{ij} - 1 = \frac{\mathbf{1}^T X^{(k)} \mathbf{1}}{n} - 1.
\end{equation}
Normalization by \( n \) ensures applicability to matrices of any size. %
A detailed derivation of the convergence condition is provided in the appendix. Figure~\ref{Fig.convergence} illustrates this process. 

\subsection{Number of iterations}
The SDSN is summarized in Algorithm \ref{alg.1}. We analyze the influence of the parameter \(\theta\) on the number of SDSN iterations in the following theorem, which shows that the iteration count grows proportionally with the value of \(\theta\).

\begin{theorem}
For a $\mathbf{X} \in \mathbb{R}^{n \times n}$ with $\max(\mathbf{X})=1$, the SDSN algorithm requires 
$$ \left\lceil \frac{\ln\left( \frac{\epsilon}{\theta\| \mathbf{X} \|_F + n} \right)}{\ln(c)} \right\rceil $$ 
iterations to produces a solution $\mathbf{X}^*$ that satisfies $\| \mathbf{X}^* - D_X^\theta \|_F < \epsilon$ where $D_X^\theta$ is the exact solution and $c \in (0, 1)$ is the convergence rate constant of the DSN algorithm.
\end{theorem}
\subsection{Robustness to truncation residual}
We observe that truncation residual vanishes over iterations in SDSN, making it ideal for mixed-precision acceleration. This technique combines high- and low-precision computations (e.g., 16-bit and 32-bit) to reduce memory usage and computation time on modern hardware like GPUs without sacrificing quality. It is especially effective for large-scale problems by identifying computations that tolerate lower precision. We provide a theoretical analysis of stability and accuracy under mixed-precision settings and introduce an acceleration strategy with formal guarantees.

\begin{theorem}[Vanishing Truncation Residual in SDSN Iterations]
Let \( X_k = \hat{X}_k + \Delta X_k \) be the decomposition of the iterate at step \( k \), where \( \Delta X_k \) is the truncated residual. As the algorithm proceeds, this truncation residual is progressively corrected and vanishes in later iterations.
\end{theorem}
Such a theorem ensures that both the gradient matrix computation in \eqref{iter.DSPFP} and SDSN can achieve computational efficiency in low-precision arithmetic while maintaining accuracy equivalent to high-precision implementations.


\section{Matching algorithm}
We propose an algorithm, Frobenius-Regularized Assignment Matching (FRAM), which approximates the QAP via a sequence of FRA, as formulated in \eqref{iter.DSPFP}. Each FRA is efficiently solved using the scalable SDSN solver. The overall procedure of FRAM is summarized in Algorithm \ref{ag.FRAM}.


        \textbf{Complexity} For $n =\tilde{n}$, Step 2-3 require $O(n^2)$ operations. Step 5 requires $O(n^3)$ operations per iteration regardless of fast and sparse matrix computation. Step 6 requires $O({n^2})$ operations. Step 10 transforms the doubly stochastic matrix $N$ back to a matching matrix $M$ using the Hungarian method \cite{kuhn1955hungarian}, which requires $O(n^3)$ operations in the worst-case scenario. However, the practical cost is far less than this, as $N$ is usually sparse. In conclusion, this algorithm has time complexity $O(n^3)$ per iteration and space complexity $O(n^2)$. 

        \textbf{Mixed-precision architecture in GPU.} The numerical precision specifications for each algorithmic operation are documented in inline code annotations. Steps 2-3 perform matrix scaling operations to precondition the input data, enabling substantial computational acceleration in Steps 5 and 6 through low-precision arithmetic implementation. The subsequent steps are conducted in double precision to compensate for accuracy degradation. Further implementation details are provided in the appendix. 

\noindent
\begin{minipage}[t]{0.52\textwidth}
    \begin{algorithm}[H]
        \caption{Frobenius-Regularized Assignment Matching (FRAM)}
        \begin{algorithmic}[1] 
            \Require $A,\tilde{A},K,\lambda,\alpha,\theta,\delta_{th}$
            \State Initial $X^{(0)} = \mathbf{0}_{n \times \tilde{n}}$
            \State {$c=\max(A,{\tilde{A}}, K)$} \hfill \Comment{FP64}
            \State {$A = A/ \sqrt{c},\ \tilde{A} =  \tilde{A}/ \sqrt{c},\  K = K/c$}\hfill \Comment{FP64}
            \While{$\delta^{(t)} > \delta_{th}$}
                \State $X^{(t)} = AN^{(t-1)}\tilde{A}+\lambda K$ \hfill \Comment{TF32}
                \State $D^{(t)}= \text{SDSN}(X^{(t)},\theta)$ \hfill \Comment{FP32}
                \State $N^{(t)} =  (1-\alpha) N^{(t-1)} + \alpha D^{(t)}$ \hfill \Comment{FP64}
                \State$\delta^{(t)} = \|N^{(t)} - N^{(t-1)}\|_{F}/\|N^{(t)}\|_{F}$ \hfill \Comment{FP64}
            \EndWhile
            \State $Discretize$ $N$ $to$ $obtain$ $M$ \hfill \Comment{FP64}
            \State \Return Matching matrix $M$
        \end{algorithmic}
        \label{ag.FRAM}
    \end{algorithm}
\end{minipage}
\hfill
\begin{minipage}[t]{0.44\textwidth}
    \begin{algorithm}[H]
        \caption{Scaling Doubly Stochastic Normalization (SDSN)}
        \begin{algorithmic}[1]
            \Require Matrix $ X $, $\theta, \gamma_{th}$
            \State $X^{(0)} = \frac{\theta}{2} X/\max(X)$ 
            \While{$\gamma^{(k)} > \gamma_{th}$}
                \State $\bar{X}^{(k)} = \frac{\mathbf{1}^T X \mathbf{1}}{n^2}$
                \State $X_1^{(k)}=\left( \frac{\mathbf{I}}{n} + \bar{X}^{(k)} I - \frac{X^{(k)}}{n} \right) \mathbf{1}\mathbf{1}^T$
                \State $\tilde{X}^{(k+1)} = X^{(k)} + X_1^{(k)} - \frac{\mathbf{1} \mathbf{1}^T X^{(k)}}{n}$
                \State $X^{(k+1)} = \max(0, \tilde{X}^{(k+1)})$
                \State $\gamma^{(k)} = n\bar{X}^{(k)} -1$
            \EndWhile
            \State \textbf{Output:} Doubly stochastic matrix $ X $
        \end{algorithmic}
        \label{alg.1}
    \end{algorithm}
\end{minipage}


\section{Experiments}
We evaluate the proposed algorithm FRAM and other contributions from the following aspects:
\begin{itemize}
    \item Q1. Compared to baseline methods, what advancements does FRAM offer in attributed graph matching tasks?
    \item Q2. How robust is FRAM in attribute-free graph matching tasks?
    \item Q3. How does mixed-precision architecture accelerate FRAM?
\end{itemize}
\textbf{Setting.} For FRAM, we set $\theta = 2$ for dense graph matching tasks and $\theta = 10$ for sparse graph matching tasks. Concerning the regularization parameter $\lambda$, according to \cite{lu2016fast}, the result is not sensitive to $\lambda$. For simplicity, we always use $\lambda =1$ in this paper.  We configure the $\alpha$ as 0.95 to align with the parameter settings in DSPFP \cite{lu2016fast}. All algorithmic comparison experiments are conducted using Python 3 on an Intel Core i7 2.80 GHz PC. All numerical computations are conducted in double precision (FP64) to ensure numerical stability, particularly for algorithms like ASM and GA that involve exponential operations and are sensitive to floating-point precision. For the evaluation of mixed-precision architecture, we utilize a dedicated hardware platform equipped with an Intel Core i9-14900 3.20 GHz CPU and an NVIDIA RTX 4080 SUPER GPU.

\textbf{Criteria.} For attributed graph matching tasks, we evaluate the accuracy of algorithms by

\begin{equation}
    \frac{1}{2}\left\|A-M \widetilde{A} M^{T}\right\|_{F}+\left\|F-M \tilde{F}\right\|_{F}.
    \label{eq:matching error}
\end{equation}
This formulation is mathematically equivalent to the original objective function \eqref{eq.Object}, but offers a more intuitive interpretation.
For graphs with only edge attribute matrices, the measurement only contains the first term of \eqref{eq:matching error}. For attribute-free graph matching tasks, the measurement is $\frac{n_c}{n}$ where $n_c$ represents the number of correct matching nodes.

\textbf{Baselines} include project fixed-point algorithms such as DSPFP \cite{lu2016fast} and AIPFP \cite{leordeanu2009integer,lu2016fast}; softassign-based algorithms such as GA \cite{gold1996graduated} (based on \eqref{eq.Object}) and ASM \cite{shen2024adaptive}; optimal transport methods such as GWL \cite{xu2019gromov} and S-GWL \cite{xu2019scalable}; and a spectral-based algorithm, GRASP \cite{hermanns2023grasp}. Optimal transport methods and GRASP are designed for attribute-free graph matching tasks. Many state-of-the-art algorithms, including Path Following \cite{zaslavskiy2008path}, FGM \cite{zhou2015factorized}, RRWM \cite{2010Reweighted}, PM \cite{egozi2012probabilistic}, BGM \cite{cour2006balanced}, and MPM \cite{cho2014finding}, do not scale well to large graphs (e.g., with over 1000 nodes), and thus are not included in our large-scale graph matching comparisons. The procedure for constructing graphs from images (Sections 7.1–7.2) is summarized in the appendix.

\begin{table}[http]
\centering
\label{tab:dataset}
\resizebox{\columnwidth}{!}{%
\begin{tabular}{lcccccc}
\hline
\textbf{Dataset} &
  $|V|$ &
  $|E|$ &
  \textbf{Attributed nodes} &
  \multicolumn{1}{l}{\textbf{Attributed edges}} &
  \textbf{Ground-truth} &
  \textbf{Dense graphs} \\ \hline
Real-world pictures &
  (700,1000) &
  (244 650, 499 500) &
  \textcolor{cyan}{\ding{51}} &
  \textcolor{cyan}{\ding{51}} &
  \textcolor{orange}{\ding{55}} &
  \textcolor{cyan}{\ding{51}} \\
CMU House &
  (600,800) &
  (179 700, 319 600) &
  \textcolor{orange}{\ding{55}} &
  \textcolor{cyan}{\ding{51}} &
  \textcolor{orange}{\ding{55}} &
  \textcolor{cyan}{\ding{51}} \\
Facebook-ego &
  4 039 &
  88 234 &
  \textcolor{orange}{\ding{55}} &
  \textcolor{orange}{\ding{55}} &
  \textcolor{cyan}{\ding{51}} &
  \textcolor{orange}{\ding{55}} \\
 \hline
\end{tabular}}
\caption{Datasets. $|V|$ is the number of nodes and  $|E|$ is the number of edges. ( , ) represents a range.}
\end{table}

\subsection{Real-world pictures}
In this set of experiments, the attributed graphs are constructed from a public dataset\footnote{\url{http://www.robots.ox.ac.uk/~vgg/research/affine/}}. This dataset, which contains eight sets of pictures, covers five common picture transformations: viewpoint changes, scale changes, image blur, JPEG compression, and illumination. 

The numerical results are presented in Table \ref{tab:real_error}. As a revolutionary version of DSPFP, FRAM achieves significant acceleration across all image sets. The average runtime of FRAM is 2.3s, compared to 6.5s for DSPFP, yielding an overall speedup of 2.8×. In addition to acceleration, FRAM consistently outperforms DSPFP regarding matching accuracy, demonstrating the effectiveness of the algorithmic design. Overall, FRAM achieves the best matching performance in more than half of the experiments while being nearly twice as fast as the second-fastest method, ASM.



\begin{table*}[htbp]
\centering
\begin{subtable}[t]{0.51\textwidth}
\centering
\resizebox{\linewidth}{!}{%
\begin{tabular}{c|clclclclclclclcl}
\hline
Performance &
  \multicolumn{16}{c}{Running Time} \\ \hline
Image Set &
  \multicolumn{2}{c|}{bark} &
  \multicolumn{2}{c|}{boat} &
  \multicolumn{2}{c|}{graf} &
  \multicolumn{2}{c|}{wall} &
  \multicolumn{2}{c|}{leuv} &
  \multicolumn{2}{c|}{tree} &
  \multicolumn{2}{c|}{ubc} &
  \multicolumn{2}{c}{bikes} \\ \hline
DSPFP & \multicolumn{2}{c|}{9.1s} & \multicolumn{2}{c|}{7.3s} & \multicolumn{2}{c|}{8.8s} & \multicolumn{2}{c|}{6.1s} & \multicolumn{2}{c|}{5.8s} & \multicolumn{2}{c|}{7.6s} & \multicolumn{2}{c|}{6.1s} & \multicolumn{2}{c}{3.3s} \\
AIPFP & \multicolumn{2}{c|}{44.3s} & \multicolumn{2}{c|}{44.2s} & \multicolumn{2}{c|}{84.4s} & \multicolumn{2}{c|}{44.8s} & \multicolumn{2}{c|}{26.3s} & \multicolumn{2}{c|}{40.3s} & \multicolumn{2}{c|}{34.3s} & \multicolumn{2}{c}{16.3s} \\
GA & \multicolumn{2}{c|}{30.8s} & \multicolumn{2}{c|}{31.0s} & \multicolumn{2}{c|}{34.2s} & \multicolumn{2}{c|}{29.7s} & \multicolumn{2}{c|}{30.8s} & \multicolumn{2}{c|}{29.8s} & \multicolumn{2}{c|}{31.8s} & \multicolumn{2}{c}{16.8s} \\
ASM & \multicolumn{2}{c|}{4.5s} & \multicolumn{2}{c|}{4.5s} & \multicolumn{2}{c|}{4.2s} & \multicolumn{2}{c|}{5s} & \multicolumn{2}{c|}{5s} & \multicolumn{2}{c|}{3.8s} & \multicolumn{2}{c|}{4.5s} & \multicolumn{2}{c}{3.2s} \\
FRAM & \multicolumn{2}{c|}{\textbf{2.6}s} & \multicolumn{2}{c|}{\textbf{2.1}s} & \multicolumn{2}{c|}{\textbf{2.4}s} & \multicolumn{2}{c|}{\textbf{2.3}s} & \multicolumn{2}{c|}{\textbf{2.2}s} & \multicolumn{2}{c|}{\textbf{2.5}s} & \multicolumn{2}{c|}{\textbf{2.4}s} & \multicolumn{2}{c}{\textbf{1.1}s} \\
\hline
\end{tabular}
}
\end{subtable}
\hfill
\begin{subtable}[t]{0.45\textwidth}
\centering
\resizebox{\linewidth}{!}{%
\begin{tabular}{c|clclclclclclclcl}
\hline
Performance &
  \multicolumn{16}{c}{Matching Error ($\times 10^4$)} \\ \hline
Image Set &
  \multicolumn{2}{c|}{bark} &
  \multicolumn{2}{c|}{boat} &
  \multicolumn{2}{c|}{graf} &
  \multicolumn{2}{c|}{wall} &
  \multicolumn{2}{c|}{leuv} &
  \multicolumn{2}{c|}{tree} &
  \multicolumn{2}{c|}{ubc} &
  \multicolumn{2}{c}{bikes} \\ \hline
DSPFP & \multicolumn{2}{c|}{5.0} & \multicolumn{2}{c|}{4.4} & \multicolumn{2}{c|}{5.1} & \multicolumn{2}{c|}{4.1} & \multicolumn{2}{c|}{5.0} & \multicolumn{2}{c|}{4.7} & \multicolumn{2}{c|}{4.0} & \multicolumn{2}{c}{4.9} \\
AIPFP & \multicolumn{2}{c|}{4.6} & \multicolumn{2}{c|}{4.5} & \multicolumn{2}{c|}{5.3} & \multicolumn{2}{c|}{4.4} & \multicolumn{2}{c|}{3.9} & \multicolumn{2}{c|}{4.7} & \multicolumn{2}{c|}{3.6} & \multicolumn{2}{c}{4.3} \\
GA & \multicolumn{2}{c|}{4.9} & \multicolumn{2}{c|}{5.3} & \multicolumn{2}{c|}{6.4} & \multicolumn{2}{c|}{6.6} & \multicolumn{2}{c|}{4.2} & \multicolumn{2}{c|}{6.3} & \multicolumn{2}{c|}{3.4} & \multicolumn{2}{c}{4.6} \\
ASM & \multicolumn{2}{c|}{4.6} & \multicolumn{2}{c|}{4.4} & \multicolumn{2}{c|}{4.9} & \multicolumn{2}{c|}{4.2} & \multicolumn{2}{c|}{\textbf{3.7}} & \multicolumn{2}{c|}{\textbf{3.5}} & \multicolumn{2}{c|}{3.3} & \multicolumn{2}{c}{\textbf{3.6}} \\
FRAM & \multicolumn{2}{c|}{\textbf{4.2}} & \multicolumn{2}{c|}{\textbf{4.0}} & \multicolumn{2}{c|}{\textbf{4.6}} & \multicolumn{2}{c|}{\textbf{3.6}} & \multicolumn{2}{c|}{4.9} & \multicolumn{2}{c|}{3.8} & \multicolumn{2}{c|}{\textbf{3.1}} & \multicolumn{2}{c}{4.4} \\
\hline
\end{tabular}
}
\end{subtable}
\caption{Performance comparison in terms of (a) running time and (b) matching error on different image sets. The number of nodes is set to 1000 (\textit{bike} set with 700 nodes). All algorithms are evaluated using double precision (FP64).}
\label{tab:real_error}
\end{table*}
\subsection{House sequence}

CMU House Sequence is a classic benchmark dataset. It consists of a sequence of images showing a toy house captured from different viewpoints. Table \ref{tab:house} illustrates that FRAM achieves the bestperformance in both speed and accuracy on the House Sequence dataset. It runs 4.1× faster than DSPFP and 3.4× faster than ASM, while attaining the lowest matching error, outperforming DSPFP by 10.1\%. These results highlight FRAM’s efficiency and effectiveness.

\begin{table}[http]
\centering
\resizebox{0.55\textwidth}{!}{%
\begin{tabular}{lcccccc}
\hline
               & DSPFP & AIPFP  & ASM   & GA     & FRAM  \\ \hline
Running time   & 3.83s & 16.75s & 3.16s & 11.86s & \textbf{0.93s}   \\
Matching error & 8117  & 9142   & 7341  & 8004   & \textbf{7047}  \\ \hline
\end{tabular}%
}
\caption{Comparisons between algorithms on graphs from the house sequence. All algorithms are evaluated using double precision (FP64).}
\label{tab:house}
\end{table}

\subsection{Social Networks}
The social network comprising `circles' (or `friends lists') from Facebook \cite{snapnets} contains 4039 users (nodes) and 88234 relations (edges). We compare different methods in matching networkswith noisy versions $5\%$, $15\%$ and $25\%$. Table \ref{tab:facebook} shows that FRAM achieves the highest node accuracy across all noise levels while maintaining computational efficiency. FRAM achieves 4\% higher accuracy than ASM (the second-best method) while operating at 2× faster computational speed. Although FRAM is slightly slower than DSPFP, it offers a substantial 15\% improvement in accuracy, demonstrating a favorable trade-off between precision and efficiency.

\begin{table}[ht]
\centering
\resizebox{0.55\textwidth}{!}{%
\begin{tabular}{c|ccccccc}
\hline
Social network & \multicolumn{2}{c}{5\% noise} & \multicolumn{2}{c}{15\% noise} & \multicolumn{2}{c}{25\% noise} \\ \hline
Methods & Acc & Time & Acc & Time & Acc & Time \\ \hline
S-GWL & 26.4\% & 1204.1s & 18.3\% & 1268.2s & 17.9\% & 1295.8s \\
GWL & 78.1\% & 3721.6s & 68.4\% & 4271.3s & 60.8\% & 4453.9s \\
DSPFP & 79.7\% & 151.3s & 68.3\% & 154.2s & 62.2\% & 156.9s \\
GA & 35.5\% & 793.2s & 21.4\% & 761.7s & 16.0\% & 832.6s \\
GRASP & 37.9\% & \textbf{63.6s} & 20.3\% & \textbf{67.4s} & 15.7\% & \textbf{71.3s} \\
ASM & 91.1\% & 387.2s & 88.4\% & 391.7s & 85.7\% & 393.1s \\
AIPFP & 68.6\% & 2705.5s & 55.1\% & 2552.7s & 47.8\% & 2513.8s \\
FRAM & \textbf{94.7\%} & 211.1s & \textbf{91.1\%} & 221.6s & \textbf{89.5\%} & 222.9s \\
\hline
\end{tabular}}
\caption{Comparisons on the Facebook networks. All algorithms are evaluated using double precision (FP64).}\label{tab:facebook}
\end{table}

\subsection{Mixed-precision acceleration}
This subsection analyzes the acceleration performance of mixed-precision architecture in FRAM across varying tasks. As demonstrated in Figure \ref{fig:mix}, the architecture shows markedly higher acceleration ratios for large-scale problems. Specifically, in the ubc(2000) matching task, mixed-precision architecture enables: (a) 12.7× speedup over standard GPU-FP64 implementations, (b) 371.4× acceleration compared to CPU-FP64 baselines. Conversely, tasks with at most 1,000 nodes exhibit sub-4× speed gains, likely attributed to limited scale that prevents full utilization of the hardware's capabilities. These observations align with Amdahl's law: fixed computational overheads dominate runtime at small scales, significantly reducing achievable performance improvements.

\begin{figure}[http]
    \centering
    \includegraphics[width=0.5\linewidth]{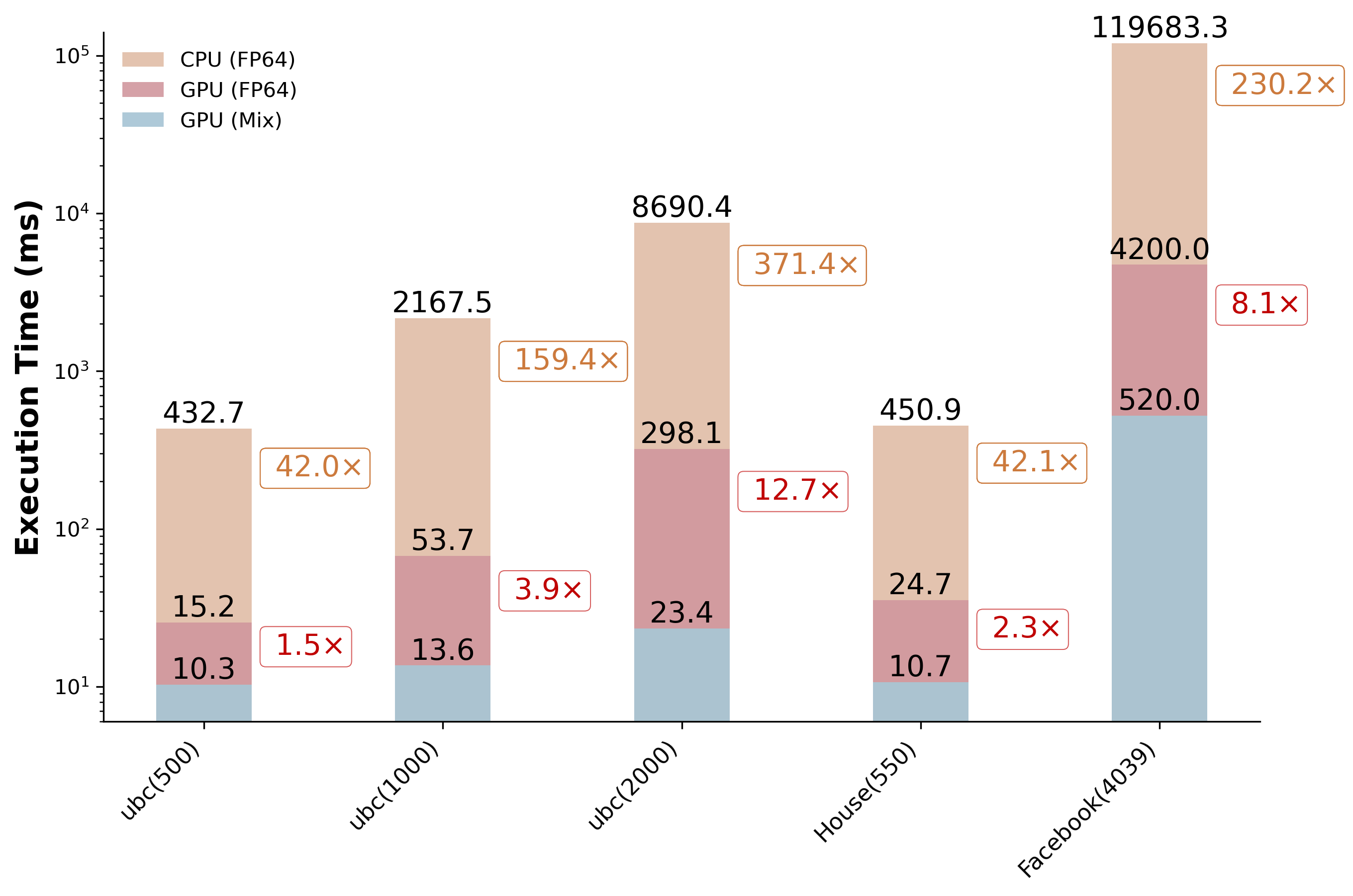}
\caption{Acceleration performance of FRAM’s mixed-precision architecture. \texttt{ubc(2000)} indicates matching of 2000-node graphs from the UBC image set.}
    \label{fig:mix}
\end{figure}



\section{Conclusion}
\label{sec:conclusion}
In the context of graph matching, this paper explores the bias introduced by projection-based relaxations. To mitigate this issue, we reformulate the projection step as a regularized linear assignment problem, providing a principled way to control the relaxation error. Building on this formulation, we propose a robust algorithm that demonstrates competitive accuracy while offering substantial speed advantages over existing baselines, including a significant improvement over the second-best method. On the computational side, we propose a theoretically grounded mixed-precision architecture. To the best of our knowledge, this is the first such design in the graph matching domain. It achieves significant acceleration while preserving numerical stability.

A limitation of this study lies in the empirical choice of the parameter $\theta$, so we plan to develop an adaptive parameter selection strategy in future work. While our framework validates the effectiveness of mixed-precision computation, its computational efficiency remains improvable. Future work may explore low-level compilation techniques to further optimize the implementation and unlock additional speed gains.

\clearpage
\bibliographystyle{plainnat}
\bibliography{ref.bib}

\begin{thebibliography}{40}
\providecommand{\natexlab}[1]{#1}
\providecommand{\url}[1]{\texttt{#1}}
\expandafter\ifx\csname urlstyle\endcsname\relax
  \providecommand{\doi}[1]{doi: #1}\else
  \providecommand{\doi}{doi: \begingroup \urlstyle{rm}\Url}\fi

\bibitem[Cho et~al.(2010)Cho, Lee, and Lee]{2010Reweighted}
Minsu Cho, Jungmin Lee, and Kyoung~Mu Lee.
\newblock {Reweighted Random Walks for Graph Matching}.
\newblock In \emph{European Conference on Computer Vision}, 2010.

\bibitem[Cho et~al.(2014)Cho, Sun, Duchenne, and Ponce]{cho2014finding}
Minsu Cho, Jian Sun, Olivier Duchenne, and Jean Ponce.
\newblock Finding matches in a haystack: A max-pooling strategy for graph matching in the presence of outliers.
\newblock In \emph{Proceedings of the IEEE Conference on Computer Vision and Pattern Recognition}, pages 2083--2090, 2014.

\bibitem[Connolly et~al.(2021)Connolly, Higham, and Mary]{connolly2021stochastic}
Michael~P Connolly, Nicholas~J Higham, and Theo Mary.
\newblock Stochastic rounding and its probabilistic backward error analysis.
\newblock \emph{SIAM Journal on Scientific Computing}, 43\penalty0 (1):\penalty0 A566--A585, 2021.

\bibitem[Cour et~al.(2006)Cour, Srinivasan, and Shi]{cour2006balanced}
Timothee Cour, Praveen Srinivasan, and Jianbo Shi.
\newblock {Balanced Graph Matching}.
\newblock \emph{Advances in Neural Information Processing Systems}, 19, 2006.

\bibitem[Egozi et~al.(2012)Egozi, Keller, and Guterman]{egozi2012probabilistic}
Amir Egozi, Yosi Keller, and Hugo Guterman.
\newblock A probabilistic approach to spectral graph matching.
\newblock \emph{IEEE Transactions on Pattern Analysis and Machine Intelligence}, 35\penalty0 (1):\penalty0 18--27, 2012.

\bibitem[Emmert-Streib et~al.(2016)Emmert-Streib, Dehmer, and Shi]{emmert2016fifty}
Frank Emmert-Streib, Matthias Dehmer, and Yongtang Shi.
\newblock Fifty years of graph matching, network alignment and network comparison.
\newblock \emph{Information sciences}, 346:\penalty0 180--197, 2016.

\bibitem[Fu et~al.(2021)Fu, Liu, Luo, and Wang]{fu2021robust}
Kexue Fu, Shaolei Liu, Xiaoyuan Luo, and Manning Wang.
\newblock Robust point cloud registration framework based on deep graph matching.
\newblock In \emph{Proceedings of the IEEE/CVF conference on computer vision and pattern recognition}, pages 8893--8902, 2021.

\bibitem[Garey and Johnson(1979)]{garey1979computers}
Michael~R Garey and David~S Johnson.
\newblock \emph{{Computers and Intractability: A Guide to the} {Theory of NP-Completeness}}.
\newblock WH Freeman \& Co., 1979.

\bibitem[Gold and Rangarajan(1996)]{gold1996graduated}
Steven Gold and Anand Rangarajan.
\newblock A graduated assignment algorithm for graph matching.
\newblock \emph{IEEE Transactions on Pattern Analysis and Machine Intelligence}, 18\penalty0 (4):\penalty0 377--388, 1996.

\bibitem[Gordon et~al.(2020)Gordon, Hiatt, Bouhaddou, Rezelj, Ulferts, Braberg, Jureka, Obernier, Guo, Batra, et~al.]{gordon2020comparative}
David~E Gordon, Joseph Hiatt, Mehdi Bouhaddou, Veronica~V Rezelj, Svenja Ulferts, Hannes Braberg, Alexander~S Jureka, Kirsten Obernier, Jeffrey~Z Guo, Jyoti Batra, et~al.
\newblock Comparative host-coronavirus protein interaction networks reveal pan-viral disease mechanisms.
\newblock \emph{Science}, 370\penalty0 (6521):\penalty0 eabe9403, 2020.

\bibitem[He et~al.(2021)He, Huang, Wang, and Zhang]{he2021learnable}
Jiawei He, Zehao Huang, Naiyan Wang, and Zhaoxiang Zhang.
\newblock Learnable graph matching: Incorporating graph partitioning with deep feature learning for multiple object tracking.
\newblock In \emph{Proceedings of the IEEE/CVF conference on computer vision and pattern recognition}, pages 5299--5309, 2021.

\bibitem[Henry et~al.(2019)Henry, Tang, and Heinecke]{henry2019leveraging}
Greg Henry, Ping Tak~Peter Tang, and Alexander Heinecke.
\newblock Leveraging the bfloat16 artificial intelligence datatype for higher-precision computations.
\newblock In \emph{2019 IEEE 26th Symposium on Computer Arithmetic (ARITH)}, pages 69--76. IEEE, 2019.

\bibitem[Hermanns et~al.(2023)Hermanns, Skitsas, Tsitsulin, Munkhoeva, Kyster, Nielsen, Bronstein, Mottin, and Karras]{hermanns2023grasp}
Judith Hermanns, Konstantinos Skitsas, Anton Tsitsulin, Marina Munkhoeva, Alexander Kyster, Simon Nielsen, Alexander~M Bronstein, Davide Mottin, and Panagiotis Karras.
\newblock {GRASP: Scalable Graph Alignment by Spectral Corresponding Functions}.
\newblock \emph{ACM Transactions on Knowledge Discovery from Data}, 17\penalty0 (4):\penalty0 1--26, 2023.

\bibitem[Higham and Mary(2022)]{higham2022mixed}
Nicholas~J Higham and Theo Mary.
\newblock Mixed precision algorithms in numerical linear algebra.
\newblock \emph{Acta Numerica}, 31:\penalty0 347--414, 2022.

\bibitem[Kuhn(1955)]{kuhn1955hungarian}
Harold~W Kuhn.
\newblock {The Hungarian method for the assignment problem}.
\newblock \emph{Naval Research Logistics Quarterly}, 2\penalty0 (1-2):\penalty0 83--97, 1955.

\bibitem[Lan et~al.(2022)Lan, Ma, Yu, Yuan, and Ma]{lan2022aednet}
Zixun Lan, Ye~Ma, Limin Yu, Linglong Yuan, and Fei Ma.
\newblock {AEDNet: Adaptive Edge-Deleting Network For Subgraph Matching}.
\newblock \emph{Pattern Recognition}, page 109033, 2022.

\bibitem[Lan et~al.(2024)Lan, Hong, Ma, and Ma]{lan2022more}
Zixun Lan, Binjie Hong, Ye~Ma, and Fei Ma.
\newblock {More Interpretable Graph Similarity Computation Via Maximum Common Subgraph Inference}.
\newblock \emph{IEEE Transactions on Knowledge and Data Engineering}, pages 1--12, 2024.

\bibitem[Leordeanu and Hebert(2005)]{leordeanu2005spectral}
Marius Leordeanu and Martial Hebert.
\newblock A spectral technique for correspondence problems using pairwise constraints.
\newblock In \emph{Tenth IEEE International Conference on Computer Vision (ICCV'05) Volume 1}, volume~2, pages 1482--1489. IEEE, 2005.

\bibitem[Leordeanu et~al.(2009)Leordeanu, Hebert, and Sukthankar]{leordeanu2009integer}
Marius Leordeanu, Martial Hebert, and Rahul Sukthankar.
\newblock An integer projected fixed point method for graph matching and map inference.
\newblock In \emph{Advances in Neural Information Processing Systems}, pages 1114--1122, 2009.

\bibitem[Leskovec and Krevl(2014)]{snapnets}
Jure Leskovec and Andrej Krevl.
\newblock {SNAP Datasets}: {Stanford} large network dataset collection.
\newblock \url{http://snap.stanford.edu/data}, June 2014.

\bibitem[Liu et~al.(2024)Liu, Feng, Xue, Wang, Wu, Lu, Zhao, Deng, Zhang, Ruan, et~al.]{liu2024deepseek}
Aixin Liu, Bei Feng, Bing Xue, Bingxuan Wang, Bochao Wu, Chengda Lu, Chenggang Zhao, Chengqi Deng, Chenyu Zhang, Chong Ruan, et~al.
\newblock Deepseek-v3 technical report.
\newblock \emph{arXiv preprint arXiv:2412.19437}, 2024.

\bibitem[Liu et~al.(2022)Liu, Lou, Wang, Xi, Shen, and Yan]{liu22k}
Chang Liu, Chenfei Lou, Runzhong Wang, Alan~Yuhan Xi, Li~Shen, and Junchi Yan.
\newblock Deep neural network fusion via graph matching with applications to model ensemble and federated learning.
\newblock In Kamalika Chaudhuri, Stefanie Jegelka, Le~Song, Csaba Szepesvari, Gang Niu, and Sivan Sabato, editors, \emph{Proceedings of the 39th International Conference on Machine Learning}, volume 162 of \emph{Proceedings of Machine Learning Research}, pages 13857--13869. PMLR, 17--23 Jul 2022.

\bibitem[Lowe(2004)]{lowe2004distinctive}
David~G Lowe.
\newblock Distinctive image features from scale-invariant keypoints.
\newblock \emph{International Journal of Computer Vision}, 60\penalty0 (2):\penalty0 91--110, 2004.

\bibitem[Lu et~al.(2016)Lu, Huang, and Liu]{lu2016fast}
Yao Lu, Kaizhu Huang, and Cheng-Lin Liu.
\newblock A fast projected fixed-point algorithm for large graph matching.
\newblock \emph{Pattern Recognition}, 60:\penalty0 971--982, 2016.

\bibitem[Narang et~al.(2017)Narang, Diamos, Elsen, Micikevicius, Alben, Garcia, Ginsburg, Houston, Kuchaiev, Venkatesh, et~al.]{narang2017mixed}
Sharan Narang, Gregory Diamos, Erich Elsen, Paulius Micikevicius, Jonah Alben, David Garcia, Boris Ginsburg, Michael Houston, Oleksii Kuchaiev, Ganesh Venkatesh, et~al.
\newblock Mixed precision training.
\newblock In \emph{Int. Conf. on Learning Representation}, 2017.

\bibitem[Neumann(1932)]{vonNeumann1932FunctionalOperatorsII}
John~Von Neumann.
\newblock \emph{{Functional Operators, Volume II: The Geometry of Orthogonal Spaces}}.
\newblock Princeton University Press, Princeton, NJ, 1932.
\newblock Reprinted in 1950.

\bibitem[Nguyen et~al.(2024)Nguyen, Diep, Nguyen, Le, Nguyen, Nguyen, Nguyen, Ho, Xie, Wattenhofer, Zhou, Sonntag, and Niepert]{nguyen2024logramedlongcontextmultigraph}
Duy M.~H. Nguyen, Nghiem~T. Diep, Trung~Q. Nguyen, Hoang-Bao Le, Tai Nguyen, Tien Nguyen, TrungTin Nguyen, Nhat Ho, Pengtao Xie, Roger Wattenhofer, James Zhou, Daniel Sonntag, and Mathias Niepert.
\newblock {LoGra-Med: Long Context Multi-Graph Alignment for Medical Vision-Language Model}, 2024.
\newblock URL \url{https://arxiv.org/abs/2410.02615}.

\bibitem[Sahni and Gonzalez(1976)]{sahni1976p}
Sartaj Sahni and Teofilo Gonzalez.
\newblock P-complete approximation problems.
\newblock \emph{Journal of the ACM (JACM)}, 23\penalty0 (3):\penalty0 555--565, 1976.

\bibitem[Shen et~al.(2020)Shen, Niu, and Zhu]{shen2020fabricated}
Binrui Shen, Qiang Niu, and Shengxin Zhu.
\newblock {Fabricated Pictures Detection with Graph Matching}.
\newblock In \emph{Proceedings of the 2020 2nd Asia Pacific Information Technology Conference}, pages 46--51, 2020.

\bibitem[Shen et~al.(2024)Shen, Niu, and Zhu]{shen2024adaptive}
Binrui Shen, Qiang Niu, and Shengxin Zhu.
\newblock {Adaptive Softassign via Hadamard-Equipped Sinkhorn}.
\newblock In \emph{Proceedings of the IEEE/CVF Conference on Computer Vision and Pattern Recognition}, pages 17638--17647, 2024.

\bibitem[Shen et~al.(2025)Shen, Niu, and Zhu]{shen2025lightning}
Binrui Shen, Qiang Niu, and Shengxin Zhu.
\newblock Lightning graph matching.
\newblock \emph{Journal of Computational and Applied Mathematics}, 454:\penalty0 116189, 2025.

\bibitem[Song et~al.(2023)Song, Wei, Bai, Yang, and Jia]{song2023graphalign}
Ziying Song, Haiyue Wei, Lin Bai, Lei Yang, and Caiyan Jia.
\newblock {Graphalign: Enhancing accurate feature alignment by graph matching for multi-modal 3d object detection}.
\newblock In \emph{Proceedings of the IEEE/CVF International Conference on Computer Vision}, pages 3358--3369, 2023.

\bibitem[Tian et~al.(2012)Tian, Yan, Zhang, Zhang, Yang, and Zha]{tian2012convergence}
Yu~Tian, Junchi Yan, Hequan Zhang, Ya~Zhang, Xiaokang Yang, and Hongyuan Zha.
\newblock {On the convergence of graph matching: Graduated assignment revisited}.
\newblock In \emph{Computer Vision--ECCV 2012: 12th European Conference on Computer Vision, Florence, Italy, October 7-13, 2012, Proceedings, Part III 12}, pages 821--835. Springer, 2012.

\bibitem[Xu et~al.(2019{\natexlab{a}})Xu, Luo, and Carin]{xu2019scalable}
Hongteng Xu, Dixin Luo, and Lawrence Carin.
\newblock Scalable gromov-wasserstein learning for graph partitioning and matching.
\newblock In \emph{Advances in Neural Information Processing Systems}, volume~32, 2019{\natexlab{a}}.

\bibitem[Xu et~al.(2019{\natexlab{b}})Xu, Luo, Zha, and Duke]{xu2019gromov}
Hongteng Xu, Dixin Luo, Hongyuan Zha, and Lawrence~Carin Duke.
\newblock Gromov-wasserstein learning for graph matching and node embedding.
\newblock In \emph{International conference on machine learning}, pages 6932--6941. PMLR, 2019{\natexlab{b}}.

\bibitem[Xu et~al.(2019{\natexlab{c}})Xu, Wang, Yu, Feng, Song, Wang, and Yu]{xu2019cross}
Kun Xu, Liwei Wang, Mo~Yu, Yansong Feng, Yan Song, Zhiguo Wang, and Dong Yu.
\newblock {Cross-lingual Knowledge Graph Alignment via Graph Matching Neural Network}.
\newblock In \emph{Proceedings of the 57th Annual Meeting of the Association for Computational Linguistics}, pages 3156--3161, 2019{\natexlab{c}}.

\bibitem[Yan et~al.(2016)Yan, Yin, Lin, Deng, Zha, and Yang]{yan2016short}
Junchi Yan, Xu-Cheng Yin, Weiyao Lin, Cheng Deng, Hongyuan Zha, and Xiaokang Yang.
\newblock A short survey of recent advances in graph matching.
\newblock In \emph{Proceedings of the 2016 ACM on international conference on multimedia retrieval}, pages 167--174, 2016.

\bibitem[Zaslavskiy et~al.(2008)Zaslavskiy, Bach, and Vert]{zaslavskiy2008path}
Mikhail Zaslavskiy, Francis Bach, and Jean-Philippe Vert.
\newblock A path following algorithm for the graph matching problem.
\newblock \emph{IEEE Transactions on Pattern Analysis and Machine Intelligence}, 31\penalty0 (12):\penalty0 2227--2242, 2008.

\bibitem[Zass and Shashua(2006)]{zass2006doubly}
Ron Zass and Amnon Shashua.
\newblock Doubly stochastic normalization for spectral clustering.
\newblock \emph{Advances in Neural Information Processing Systems}, 19, 2006.

\bibitem[Zhou and De~la Torre(2015)]{zhou2015factorized}
Feng Zhou and Fernando De~la Torre.
\newblock {Factorized Graph Matching}.
\newblock \emph{IEEE transactions on pattern analysis and machine intelligence}, 38\penalty0 (9):\penalty0 1774--1789, 2015.

\end{thebibliography}
\clearpage
\appendix

\clearpage
\setcounter{lemma}{0} 
\setcounter{theorem}{0}
\setcounter{corollary}{0}
\section{Notations}
\label{sec:appendix_section}
The common symbols are summarized in Table \ref{tab:symbols}.
\begin{table}[h]
\centering
\caption{Symbols and Notations.}
\resizebox{0.6\columnwidth}{!}{
\begin{tabular}{cc}
\hline
Symbol                       & Definition                         \\ \hline
${G},\tilde{{G}}$                 & matching graphs                   \\
$A,\tilde{A}$                          & edge attribute matrices of  ${G}$ and $\tilde{{G}}$                \\
$F,\tilde{F}$                          & node attribute matrices of  ${G}$ and $\tilde{{G}}$          \\
$n,\tilde{n}$                          & number of nodes of  ${G}$ and $\tilde{{G}}$                  \\
$M$                          & matching matrix                   \\
$\Pi_{n \times n}$           & the set of $n \times n$ permutation matrices       \\
$\mathcal{D}_{n \times n}$        & the set of $n \times n$ doubly stochastic matrices \\\hline
$\mathbf{1}$,$ \mathbf{0}$   & a column vector of all 1s,0s      \\ 
$\operatorname{tr}(\cdot)$                  & trace                             \\
$\langle \cdot,\cdot\rangle$ & Frobenius inner product                     \\
$\|\cdot \|_{F}$                    &Frobenius norm                         \\
$\theta$                      & the parameter in FRA      \\
$\alpha$                      & the step size parameter      
\\  \hline
FP8/16/32/64  & 8/16/32/64-bit Floating Point
\\
TF32   & TensorFloat 32\\
BF16    &16-bit Brain Floating Point
 \\\hline
\end{tabular}%
}
\end{table}
\label{tab:symbols}

\section{Convergence Criterions}
\subsection{Convergence Criterion of SDSN}
We recall the SDSN projection steps below:
$$
\tilde{X}^{(k)} =\mathcal{P}_1 \left( X^{(k-1)} \right), \ \  X^{(k)} =\mathcal{P}_2 \left(\tilde{X}^{(k)}  \right),
$$
$$
\mathcal{P}_1(X) =X + \left( \frac{I}{n} + \frac{\mathbf{1}^T X \mathbf{1}}{n^2}I - \frac{\mathbf{X}}{n} \right) \mathbf{1}\mathbf{1}^T - \frac{\mathbf{\mathbf{1} \mathbf{1}^T X}}{n} , \quad \quad \mathcal{P}_2(X) = \frac{X + |X|}{2}.
$$
Before the non-negative projection, the updated matrix satisfies
\begin{equation}
\sum_{i,j} \tilde{X}^{(k)}_{ij} =\sum_{\tilde{X}^{(k)}_{ij}>0} \tilde{X}^{(k)}_{ij}+\sum_{\tilde{X}^{(k)}_{ij}<0} \tilde{X}^{(k)}_{ij}= n.
\end{equation}
After applying the non-negative projection (all negative elements are set to 0), we obtain
\begin{align}
    \sum_{i,j} X^{(k)}_{ij} = 
    \sum_{\tilde{X}^{(k)}_{ij}>0} \tilde{X}^{(k)}_{ij} =  n+ \sum_{\tilde{X}^{(k)}_{ij}<0} |\tilde{X}^{(k)}_{ij}|.
\end{align}
Since $ \sum_{i,j} X^{(k)}_{ij}$ converges to $n$, 
\begin{equation}
\sum_{\tilde{X}^{(k)}_{ij}<0} |\tilde{X}^{(k)}_{ij}| = \sum_{i,j} X^{(k)}_{ij} - n
\end{equation}
is a natural measure for the residual.
Moreover, the von Neumann successive projection lemma \cite{vonNeumann1932FunctionalOperatorsII} guarantees that this distance decreases monotonically over successive iterations. Therefore, we define a \textbf{dimension-invariant distance measure} as
\begin{equation}
    \gamma(X^{(k)}) = \frac{1}{n}\sum_{i,j} X^{(k)}_{ij} - 1 = \frac{\mathbf{1}^T X^{(k)} \mathbf{1}}{n} - 1,
\end{equation}
where normalization by \( n \) ensures applicability to matrices of any size. Furthermore, since \(\mathbf{1}^T X^{(k-1)} \mathbf{1}\) is computed in $\mathcal{P}_1(\cdot)$, we can avoid the computation for the distance by approximately using $\gamma(X^{(k-1)})$ in $k$th iteration.
\subsection{Convergence Criterion of the Matching algorithm}
The matching algorithm in DSPFP \cite{lu2016fast} adopts the following original convergence criterion:
\begin{equation}
    \max\left(\left|\frac{N^{(t)}}{\max(N^{(t)})} - \frac{N^{(t-1)}}{\max(N^{(t-1)})}\right|\right),
\end{equation} 
which was also adopted by the ASM \cite{shen2024adaptive}. This metric primarily captures local fluctuations between the successive iterations. To better characterize the global structural evolution, we adopt the following \textbf{normalized Frobenius criterion}:
\begin{equation}
    \delta^{(t)} = \frac{\|N^{(t)} - N^{(t-1)}\|_{F}}{\|N^{(t)}\|_{F}}.
\end{equation}
This criterion demonstrates enhanced suitability for our framework due to two principal considerations: (1) The Frobenius norm inherently aligns with the FRA and the objective formulation, ensuring mathematical consistency throughout the optimization process. (2) The normalization scheme provides a scale-invariant measurement of the variation of the solution trajectory.

\section{Proofs in Section 4}
\begin{theorem}
The solution of the doubly stochastic projection with a scaling parameter $\theta$ 
\begin{equation}
D_X^\theta = \arg\min_{D \in \mathcal{D}_{n \times n}} \|D - \frac{\theta}{2}X\|_{F}^2
\end{equation}
is the solution of a Frobenius-regularized linear assignment problem
\begin{equation}
\begin{aligned}
         D_X^\theta = \arg\max_{D \in \mathcal{D}_{n \times n}}  \Gamma^{\theta}(X)
     ,\quad \Gamma^{\theta}(X) = \langle D, X \rangle -\frac{1}{\theta}\langle D, D \rangle.
\end{aligned}
\end{equation}
\end{theorem}
\begin{proof}
The squared Frobenius norm can be expanded as:
\begin{equation}
\|D - \frac{\theta}{2}X\|_{F}^2 = \langle D, D \rangle - 2\langle D, \frac{\theta}{2}X \rangle + \langle \frac{\theta}{2}X, \frac{\theta}{2}X \rangle.
\end{equation}
Since the term \(\langle \frac{\theta}{2}X, \frac{\theta}{2}X \rangle\) is a constant with respect to \( D \), it does not affect the optimization. Thus, the objective simplifies to:
\begin{equation}
D_X^\theta = \arg\min_{D \in \mathcal{D}_{n \times n}} \langle D, D \rangle - \theta \langle D, X \rangle.
\end{equation}
As a result, the problem reduces to:
\begin{equation}
D_X^\theta = \arg\max_{D \in \mathcal{D}_{n \times n}} \langle D, X \rangle -\frac{1}{\theta}\langle D, D \rangle.
\end{equation}
\end{proof}

\begin{theorem}
For a nonnegative matrix $X \in \mathbb{R}^{n \times n}$, the following inequality holds:
    \begin{equation}
\epsilon^{\theta}_{X}\leq  \frac{1}{\theta}. 
\end{equation}
\end{theorem}
\begin{proof}
Since $D^\theta_X$ is the maximizer of $\Gamma^\theta(X)$, we have
\begin{equation}
\Gamma^\theta(D^\theta_X) \geq \Gamma^\theta(D^\infty_X),
\end{equation}
where $D^\infty_X$ is a maximizer of the unperturbed problem $\max_{D \in \mathcal{D}_{n\times n}} \langle D, X \rangle$.

Expanding the inequality gives:
\begin{equation}
\langle D^\theta_X, X \rangle - \frac{1}{\theta}\langle D^\theta_X, D^\theta_X \rangle \geq \langle D^\infty_X, X \rangle - \frac{1}{\theta}\langle D^\infty_X, D^\infty_X \rangle.
\end{equation}
Rearranging terms, we have:
\begin{equation}
\langle D^\theta_X, X \rangle - \langle D^\infty_X, X \rangle \geq \frac{1}{\theta}\bigl(\langle D^\theta_X, D^\theta_X \rangle - \langle D^\infty_X, D^\infty_X \rangle\bigr).
\end{equation}
In particular, the maximum of $\langle D^\infty_X, D^\infty_X \rangle $ is $n$ when $D^\infty_X$ is a permutation matrix. Therefore,  we have $\langle D^\infty_X, D^\infty_X \rangle \leq n$ for general cases.
\begin{equation}
\langle D^\theta_X, X \rangle - \langle D^\infty_X, X \rangle \geq \frac{1}{\theta}(\langle D^\theta_X, D^\theta_X \rangle - n).
\end{equation}

Taking the negative of both sides,
\begin{equation}
\langle D^\infty_X, X \rangle - \langle D^\theta_X, X \rangle \leq \frac{1}{\theta}(n - \langle D^\theta_X, D^\theta_X \rangle).
\end{equation}
Note that $D^\theta_X \in \mathcal{D}_{n \times n}$ is doubly stochastic, so $1 \leq \langle D^\theta_X, D^\theta_X \rangle \leq n$ (The inner product achieves its maximum value when $D^\theta_X$ is a permutation matrix, while attaining its minimum value under the condition that all elements of $D^\theta_X$ are $\frac{1}{n}$). Therefore,
\begin{equation}
\langle D^\infty_X, X \rangle - \langle D^\theta_X, X \rangle \leq \frac{n - 1}{\theta},
\end{equation}
which completes the proof.
\end{proof}

\begin{theorem}
    As $\theta \to \infty$, the matrix $D_X^\theta$ converges to the unique matrix $D^*$. \( D^* \) minimizes the regularization term within the set \( \mathcal{F} \), the convex hull of the optimal permutation matrices.
\end{theorem}
\begin{proof}
We prove the theorem in two steps: (1) $D^{\infty}_X$ lies on the face $\mathcal{F}$, and (2) $D^{\infty}_X$ must equal $D^*$.

\textbf{Step 1.} Since $D_X^{\theta}$ is defined as the maximizer of 
\begin{equation}
\Gamma^\theta(D) = \langle D, X \rangle - \frac{1}{\theta}\langle D, D \rangle,
\end{equation}
we have
\begin{equation}
\Gamma^\theta(D_X^\theta) \geq \Gamma^\theta(D^{*}),
\end{equation}
which implies
\begin{equation}
\langle D_X^\theta, X \rangle - \frac{1}{\theta}\langle D_X^\theta, D_X^\theta \rangle \ge \langle D^*, X \rangle - \frac{1}{\theta}\langle D^*, D^* \rangle.
\end{equation}
As $\theta \to \infty$, we have 
\begin{equation}
\frac{1}{\theta}\langle D_X^\theta, D_X^\theta \rangle \to 0 \quad \text{and} \quad \frac{1}{\theta}\langle D^*, D^* \rangle \to 0.
\end{equation}
Taking the limit, we get
\begin{equation}
\lim_{\theta \to \infty} \langle D_X^\theta, X \rangle \ge \langle D^{*}, X \rangle.
\end{equation}
Since $D^{*}$ is the optimal solution to the linear assignment problem, for any $D \in \mathcal{D}_{n \times n}$,
\begin{equation}
\langle D, X \rangle \le \langle D^{*}, X \rangle.
\end{equation}
In particular,
\begin{equation}
\langle D_X^\theta, X \rangle \le \langle D^{*}, X \rangle.
\end{equation}
Combining both inequalities, we obtain
\begin{equation}
\lim_{\theta \to \infty} \langle D_X^\theta, X \rangle = \langle D^{*}, X \rangle,
\end{equation}
which means $D^{\infty}_X$ lies in the face $\mathcal{F}$.

\textbf{Step 2.} For each fixed $\theta$, the objective $\Gamma^\theta(D)$ is strictly concave in $D$ due to the quadratic term $-\tfrac{1}{\theta}\langle D, D \rangle$. Strict concavity ensures that for large $\theta$, the maximizer of $\Gamma^\theta(D)$ is unique.

Suppose, for the sake of contradiction, that $D^{\infty}_X \neq D^{*}$. Since both are in $\mathcal{F}$, we have
\begin{equation}
\langle D^{\infty}_X, X \rangle = \langle D^{*}, X \rangle.
\end{equation}
If $D^{\infty}_X$ were distinct from $D^{*}$, then as $\theta \to \infty$, the perturbed objective $\Gamma^\theta(D)$ would admit at least two different maximizers ($D^{\infty}_X$ and $D^{*}$) with the same objective value. This contradicts the uniqueness guaranteed by strict concavity. Therefore, $D^{\infty}_X$ must coincide with $D^{*}$.

Since the entire sequence $D_X^\theta$ is bounded and any convergent subsequence converges to $D^{*}$, it follows that
\begin{equation}
D_X^\theta \to D^{*} \ \text{as } \theta \to \infty.
\end{equation}
This establishes the desired convergence.
\end{proof}

\begin{theorem}
    As $\theta \to 0$, the matrix $D_X^\theta$ converges to the matrix $\frac{\mathbf{1}\mathbf{1}^T}{n}$.
\end{theorem}
\begin{proof}
As $\theta \to 0$, the matrix $D_X^\theta$ exhibits the limiting behavior:
    \begin{equation}
        \lim_{\theta \to 0} D^{\theta}_X = \arg \min_{D \in \mathcal{D}_{n \times n}} \langle D,D \rangle= \sum D_{ij}^2.
    \end{equation}
Since $\sum D_{ij}=n$, $\sum D_{ij}^2$ achieve minimum when all entries are equal to $1/n$. Consequently,
\begin{equation}
\lim_{\theta \to 0} D^{\theta}_X = \frac{\mathbf{1}\mathbf{1}^T}{n}.
\end{equation}
\end{proof}

\section{Proofs in Section 5}
\begin{theorem}
    For an asymmetric square matrix $X$, \eqref{eq.P1_sol} is the solution of the projection \eqref{eq.p1}.
\end{theorem}
        \begin{proof}
        As the projection \(P_2\) is suitable for the asymmetric case, it suffices to consider \(P_1\). The Lagrangian corresponding to \(P_1\) takes the form:
             \begin{equation}
		\begin{aligned}
		L(Y, \mu_1, \mu_2) = &\operatorname{tr}\left( Y Y^T - 2 X Y^T \right) - \mu_1^T (Y \mathbf{1} - \mathbf{1}) - \mu_2^T (Y^T \mathbf{1} - \mathbf{1}).\label{L(F_L_mu)}
		\end{aligned}
	             \end{equation}	
		By differentiating (\ref{L(F_L_mu)}) , we obtain:

			\begin{equation}\frac{\partial L}{\partial Y} = Y - X - \mu_1 \mathbf{1}^T - \mathbf{1} \mu_2^T,\end{equation}
		\begin{equation}\frac{\partial L}{\partial \mu_1} = Y\mathbf{1} - \mathbf{1} ,\end{equation}
\begin{equation}\frac{\partial L}{\partial \mu_1} = Y^T\mathbf{1} - \mathbf{1} .\end{equation}

		Thus, setting \(\frac{\partial L}{\partial Y} = 0\), we have:
		\begin{equation}
			Y = X + \mu_1 \mathbf{1}^T + \mathbf{1} \mu_2^T.
		\end{equation}

		Applying \(1\) to both sides of the above equation, we get:
		\begin{equation}
			\mathbf{1} = X \mathbf{1} + n\mu_1 + \mathbf{1} \mu_2^T\bm1.\label{1_X1}
		\end{equation}
		
		Let \(n \times (\ref{1_X1}) \):
		\begin{equation}
			n\mathbf{1} = nX\mathbf{1} + n^2\mu_1 + n\mathbf{1}\mu_2^T\mathbf{1}.
		\label{n1=nX1}
        \end{equation}

		Let $11^{T} \times (\ref{1_X1})$, we get
		\begin{equation}
		n\mathbf{1} = \mathbf{1}\mathbf{1}^{T} X \mathbf{1} + n \mathbf{1}\mathbf{1}^T\mu_{1} + n\mathbf{1} \mu_{2}^T \mathbf{1}.
		\label{n1=11X1}\end{equation}
		Then, (\ref{n1=nX1}) - (\ref{n1=11X1}) can obtain
		\begin{equation}
		(nI - \mathbf{1}\mathbf{1}^T)X\mathbf{1} + n(nI - \mathbf{1}\mathbf{1}^T)\mu_1 = 0.
		\end{equation}
        
		Solving this equation, $\mu_{1} = -\frac{1}{n} X \mathbf{1} - k_{1}\mathbf{1}$, where $k_{1} \in \mathbb{R}$. Similarly, we can obtain $\mu_{2} = -\frac{1}{n} X \mathbf{1} - k_{2}\mathbf{1}, \text{where } k_{2} \in \mathbb{R}.$ Using the aforementioned conclusion, (\ref{L(F_L_mu)}) is equivalent to
			{\footnotesize\begin{align}
		L(Y, k_{1}, k_{2}) &= \operatorname{tr} \left( Y Y^{T} - 2 X Y \right) +  \left( \frac{1}{n}X \mathbf{1} k_{1} \mathbf{1} \right)^{T} (Y \mathbf{1} - \mathbf{1}) 
        \\ 
        & + \left( \frac{1}{n} X^T \mathbf{1} + k_{2} \mathbf{1} \right)^{T} (Y^T \mathbf{1} - \mathbf{1}).\end{align}}
		
		Then, we can get
        
		\begin{align}
		\frac{\partial L}{\partial Y} = Y - X + (\frac{1}{n}X \mathbf{1} + k_{1} \mathbf{1})\mathbf{1}^{T}  + \mathbf{1}(\frac{1}{n}X^T\mathbf{1} + k_{2} \mathbf{1})^T, 
		\end{align}

		\begin{equation}
		\frac{\partial L}{\partial k_1} = \mathbf{1}^{T}(Y\mathbf{1}-\mathbf{1}), \quad \frac{\partial L}{\partial k_2} = \mathbf{1}^{T}(Y^T\mathbf{1}-\mathbf{1}).
		\end{equation}

        \text{Let } \begin{equation}\frac{\partial L}{\partial Y} = \frac{\partial L}{\partial k_1} = \frac{\partial L}{\partial k_2} = 0,\end{equation} and solving these equations, we have			
			\begin{equation}k_1 + k_2 = -\frac{1}{n}(n^2 + \mathbf{1}^T X \mathbf{1}),\end{equation}
            \begin{equation}Y = X + \left(\frac{1}{n}I + \frac{\mathbf{1}^T X \mathbf{1}}{n^2}I - \frac{1}{n}X\right) \mathbf{1}\mathbf{1}^T - \frac{1}{n}\mathbf{1}\mathbf{1}^T X.\end{equation}
		\end{proof}

\begin{theorem} 
	For a $\mathbf{X} \in \mathbb{R}^{n \times n}$ with $\max(\mathbf{X})=1$, the SDSN algorithm requires 
$$ \left\lceil \frac{\ln\left( \frac{\epsilon}{\theta(\| \mathbf{X} \|_F + n)} \right)}{\ln(c)} \right\rceil $$ 
iterations to produces a solution $\mathbf{X}^*$ that satisfies $\| \mathbf{X}^* - D_X^\theta \|_F < \epsilon$ where $D_X^\theta$ is the exact solution and $c \in (0, 1)$ is the convergence rate constant of the DSN algorithm.
	\end{theorem}
	
	\begin{proof}
	\begin{equation}
	X_{k+1} = \mathcal{P}_1\mathcal{P}_2(X_k)
\end{equation}
where $\mathcal{P}_1$ and $\mathcal{P}_2$ denote projection operators maintaining the constraint $X_k\mathbf{1} = (\mathbf{1}^TX_k)^T = \mathbf{1}$. This operation can be explicitly expressed as:

\begin{equation}
	X_{k+1} = \mathcal{P}_1\left(X_k + |X_k|\right)
\end{equation}
with the absolute value operation applied element-wise to matrix entries. Expanding the projection operators yields the detailed formulation:
{\small \begin{equation}X_{k+1} = \frac{1}{2}(X_k + |X_k|) + \frac{\mathbf{1}^T [\frac{1}{2}(X_k + |X_k|)] \mathbf{1}}{n^2} \mathbf{1}\mathbf{1}^T - \frac{1}{n}[\frac{1}{2}(X_k + |X_k|)] \mathbf{1}\mathbf{1}^T - \frac{1}{n}\mathbf{1}\mathbf{1}^T [\frac{1}{2}(X_k + |X_k|)] + \frac{1}{n}\mathbf{1} \mathbf{1}^T.\end{equation}}

Through vectorization and Kronecker product analysis, we transform the matrix equation into its vector form:

\begin{equation} 
	\text{vec}(X_{k+1}) = A\,\text{vec}\left(|X_k| +  X_k\right) + \frac{1}{n}\text{vec}\left(\mathbf{1}\mathbf{1}^T\right)
\end{equation}
where $A = \frac{1}{2}(I - \frac{1}{n} \mathbf{1}\mathbf{1}^T)\otimes(I - \frac{1}{n} \mathbf{1}\mathbf{1}^T)$ possesses a spectral radius of $\frac{1}{2}$. Defining the error vector $e_k = \text{vec}(X_k) - \text{vec}(X_*)$ for solution matrix $X_*$, we derive the error propagation relationship:

\begin{equation}
	e_{k+1} = A\,\text{vec}\left(|X_k| - X_*\right) +  Ae_k
\end{equation}
Applying norm inequalities and leveraging the spectral properties of $A$, we obtain:
\begin{equation} 
	\|e_{k+1}\| \leq \frac{1}{2}\left(\text{c}_k\|e_{k}\| + \|e_{k}\|\right) = c\|e_{k}\|
\end{equation}
where $c = \sup_{1:k}\{\frac{\text{c}_k + 1}{2}\} < 1$ establishes the linear convergence rate.

Considering scaled initial conditions $\theta\mathbf{X}$, the first iteration error becomes $\|e_1\| = \|\theta \cdot e_0\|$. For $k$ iterations with scaling factor $\theta$, the error evolution follows:

\begin{equation} 
	\|e_k\| = \|\theta \cdot c^k \cdot e_0 \|
\end{equation}
The $\epsilon$-accuracy requirement $\|\theta \cdot c^{k'} \cdot e_0 \| \leq \epsilon$ leads to the logarithmic relationship:

\begin{equation} 
	\ln\theta + k' \ln c + \ln\|e_0\| \leq \ln\epsilon
\end{equation}
Solving for the required iterations yields:
\begin{equation} 
	k' \geq \frac{\ln(\epsilon/\theta) + \ln(\epsilon/\|e_0\|)}{\ln c} = \frac{\ln\left(\epsilon/(\theta\|e_0\|)\right)}{\ln c}
\end{equation}
Given the initial error bound $\|e_0\| = \|X\|_F + n \leq \|X\|_F + n$, we finalize the iteration complexity:

\begin{equation} 
	k' \geq \frac{\ln\left( \frac{\epsilon}{\theta(\| \mathbf{X} \|_F + n)} \right)}{\ln c}
\end{equation}
This demonstrates the logarithmic relationship between the scaling factor $\theta$ and the required iterations to maintain solution accuracy, completing the proof.
\end{proof}

\begin{lemma}
    The closed-form solution to the optimization problem
    \begin{equation}
    P_3(X) = \arg\min_Y \|X - Y\|_{F}^2,\  \text{s.t. } Y\mathbf{1} =  Y^T\mathbf{1} = \mathbf{0}\end{equation} 
    is given by:
    \begin{equation}
    P_3(X) = X + \left(\frac{1}{n^2}\bm1^T X\bm1 - \frac{1}{n}X\right)\bm1\bm1^T - \frac{1}{n}\bm1\bm1^T X.
    \end{equation}
\end{lemma}

\begin{proof}
The Lagrangian corresponding to the problem takes the form:
\begin{equation}
    L(Y, \mu_1, \mu_2) = \operatorname{tr}(YY^T - 2XY) - \mu_1^TY\bm1 - \mu_2^TY^T.\label{mixf}
\end{equation}

By differentiating (\ref{mixf}), we can obtain the following:
\begin{equation}
\frac{\partial L}{\partial Y} = Y - X - \mu_1 \bm1^T - \bm1 \mu_2^T, \quad
\frac{\partial L}{\partial \mu_1} = Y \bm1, \quad \frac{\partial L}{\partial \mu_2} = Y^T \bm1.
\end{equation}
Thus, let $\frac{\partial L}{\partial Y} = 0$, we have:
\begin{equation}
Y = X + \mu_1 \bm1^T + \bm1 \mu_2^T.
\end{equation}

Applying $\bm1$ to both sides of the equation above, we get:
\begin{equation}
    0 = X \bm1 + n \mu_1 + \bm1 \mu_2^T \bm1.\label{0_x}
\end{equation}
Multiplying both sides of (\ref{0_x}) by $n$, we get:
\begin{equation}
\bm0 = n X \bm1 + n^2 \mu_1 + n \bm1 \mu_2^T \bm1. \label{0_n}
\end{equation}
Multiplying both sides of (\ref{0_n}) by $\bm1\bm1^T$, we get:
\begin{equation}
\bm0 = \bm1\bm1^T X \bm1 + n \bm1\bm1^T \mu_1 + n\bm1 \mu_2^T\bm1. \label{0_11t}
\end{equation}

Subtracting (\ref{0_n}) and (\ref{0_11t}), we obtain:
\begin{equation}
(n I - \bm1 \bm1^T) X \bm1 + n(n I - \bm1 \bm1^T) \mu_1 = 0.
\end{equation}

Solving this equation, we find:
\begin{equation}
\mu_1 = \frac{1}{n} X \bm1 - k_1 \bm1, \quad \mu_2 = -\frac{1}{n} X^T \bm1 - k_2 \bm1,
\end{equation}
where $k_1, k_2 \in \mathbb{R}$.

Substituting the results into (\ref{mixf}), we rewrite the Lagrangian as:
\begin{align}
L(Y, k_1, k_2) = \operatorname{tr}(FF^T - 2XF) + (\frac{1}{n} X\bm1 + k_1\bm1)^TY\mathbf{1} +(\frac{1}{n} X^T \bm1 + k_2 \bm1)^T Y^T. 
\end{align}

Finally, setting $\frac{\partial L}{\partial Y} = \frac{\partial L}{\partial k_1} = \frac{\partial L}{\partial k_2} = 0$, we solve for $Y$:
\begin{equation}
Y = X +  \left(\frac{\bm1^T X \bm1}{n^2}I - \frac{1}{n} X \right)\bm1\bm1^T-\frac{1}{n}\bm1\bm1^TX.
\end{equation}
\end{proof}

\begin{theorem}\label{mix}
Let \( X_k = \hat{X}_k + \Delta X_k \) be the decomposition of the iterate at step \( k \), where \( \Delta X_k \) is the truncation residual. As the algorithm proceeds, this truncation error is progressively corrected and vanishes in later iterations.
\end{theorem}
\begin{proof}

In SDSN calculations, the nonnegativity-enforcing part can be computed entirely in low precision, so our truncation error analysis focuses on the first component $\mathcal{P}_1$ \eqref{eq.P1_sol}. To examine the behavior of the truncated component \(\Delta X_k\) during iterations, we decompose \(X_k\), the variable at the \(k\)th iteration, as
\begin{equation}
X_k = \hat{X}_k + \Delta X_k.
\end{equation}
With this decomposition, $\mathcal{P}_1(X_k)$ becomes
{\footnotesize \begin{align}
      & \left(\frac{I}{n}  + \frac{1^T (\hat{X}_k + \Delta X_k) 1I}{n^2}  - \frac{(\hat{X}_k + \Delta X_k)1}{n}  \right) 11^T +\hat{X}_k + \Delta X_k - \frac{1}{n} 11^T (\hat{X}_k + \Delta X_k)
    \\
    &= \underbrace{\hat{X}_k + \left(\frac{I}{n}  + \frac{1^T \hat{X}_k 1I}{n^2}  - \frac{\hat{X}_k}{n}  \right) 11^T - \frac{11^T \hat{X}_k}{n}}_{\mathcal{P}_1(\hat{X}_k)}  + \underbrace{ \Delta X_k + \left(\frac{1^T \Delta X_k 1}{n^2} I - \frac{\Delta X_k}{n}  \right) 11^T - \frac{11^T \Delta X_k}{n}  }_{\Delta T_k}.\nonumber
\end{align}}
\(\Delta T_k\) denotes the truncation error caused by the truncated component \(\Delta X_k\). The Lemma establishes that \(\Delta T_k=\mathcal{P}_3(\Delta X_k)\), thereby transforming its subsequent projection process into $\mathcal{P}_2\mathcal{P}_3 \dots \mathcal{P}_2\mathcal{P}_3(\Delta X_k)$. This truncation residual converges to $\mathbf{0}_{n\times n }$ satisfying the constraints associated with $\mathcal{P}_2(\cdot)$ and $\mathcal{P}_3(\cdot)$.
\end{proof}

\section{Implementation Details of the Mixed-Precision Architecture}
\setcounter{algorithm}{0} 
    \begin{algorithm}[H]
        \caption{Frobenius-Regularized Assignment Matching (FRAM)}
        \begin{algorithmic}[1] 
            \Require $A,\tilde{A},K,\lambda,\alpha,\theta,\delta_{th}$
            \State Initial $X^{(0)} = \mathbf{0}_{n \times \tilde{n}}$
            \State {$c=\max(A,{\tilde{A}}, K)$} \hfill \Comment{FP64}
            \State {$A = A/ \sqrt{c},\ \tilde{A} =  \tilde{A}/ \sqrt{c},\  K = K/c$}\hfill \Comment{FP64}
            \While{$\delta^{(t)} > \delta_{th}$}
                \State $X^{(t)} = AN^{(t-1)}\tilde{A}+\lambda K$ \hfill \Comment{TF32}
                \State $D^{(t)}= \text{SDSN}(X^{(t)},\theta)$ \hfill \Comment{FP32}
                \State $N^{(t)} =  (1-\alpha) N^{(t-1)} + \alpha D^{(t)}$ \hfill \Comment{FP64}
                \State$\delta^{(t)} = \|N^{(t)} - N^{(t-1)}\|_{F}/\|N^{(t)}\|_{F}$ \hfill \Comment{FP64}
            \EndWhile
            \State $Discretize$ $N$ $to$ $obtain$ $M$ \hfill \Comment{FP64}
            \State \Return Matching matrix $M$
        \end{algorithmic}
        \label{ag.FRAM1}
    \end{algorithm}

    Algorithm 1 employs mixed-precision arithmetic to enhance computational efficiency while maintaining numerical stability, with each operation's precision specified via inline annotations. The algorithm begins by normalizing matrices $A$, $\tilde{A}$, and $K$ using FP64 operations in steps 2-3 to prevent numerical overflow and ensure input consistency, which also enables subsequent low-precision acceleration. Step 5 leverages TF32 precision for critical computations, taking advantage of GPU fused multiply-add operations while guaranteeing minimal precision loss as formalized in Theorem \ref{mix}, followed by FP32 operations in step 6 to maintain consistency; step 7 then reverts to FP64 for high-precision iterative residual updates to compensate for earlier precision trade-offs, with steps 8 and 10 continuing in FP64 to ensure final output accuracy. This carefully designed precision hierarchy accelerates performance in the preconditioned computation phases while upholding the reliability required for assignment matching tasks.

\begin{table}[htbp]
    \centering
    \small
    \setlength{\tabcolsep}{6pt} 
    \begin{tabular}{ccccc}
    \toprule
    \textbf{Data Type} & \textbf{Bits} & \textbf{Range} & \textbf{Precision} & \textbf{FLOPs (RTX 4080)} \\
    \midrule
    FP32 & 32 & $ [-10^{38}, 10^{38}]$ & $10^{-6}$ & 52.2 TeraFLOPS \\
    TF32 & 19 & $ [-10^{38}, 10^{38}]$ & $10^{-3}$ & 209 TeraFLOPS \\
    FP64 & 64 & $ [-10^{308}, 10^{308}]$ & $10^{-16}$ & 0.82 TeraFLOPS \\
    \bottomrule
    \end{tabular}
    \caption{Characteristics of FP32, TF32, and FP64 floating-point formats. The listed ranges are approximate, based on the IEEE 754 standard.  The FLOPs column reports the theoretical peak performance (in TeraFLOPS) achieved by the NVIDIA RTX 4080 GPU for each format. Here, 1 TeraFLOPS equals $10^{12}$ floating-point operations per second, providing a standard measure of compute performance.}
\end{table}

\section{Graphs from Images}

\textbf{The construction} consists of three primary steps. (1) Node extraction:  SIFT \cite{lowe2004distinctive} extracts key points as potential nodes and computes the nodes' attributes; (2) Node selection: select the nodes that exhibit a high degree of similarity (i.e., the inner product of feature vectors) to all candidate nodes of the other graph. (3) Edge attribute calculation: nodes form a fully connected graph weighted by inter-node Euclidean distances.

\textbf{Real-world images.} The graphs are constructed as described in the previous paragraph. The number of nodes is set to 1000 (For the \textit{bike} set, we only record the results for the first three pictures with 700 nodes, as the other images lack sufficient keypoints.). The running time and matching error are computed by averaging the results over five matching pairs (1 vs. 2, 2 vs. 3, …, 5 vs. 6) from the same image set. 

\textbf{House sequence}\footnote{https://www.cs.cmu.edu/afs/cs/project/vision/vasc/idb/images/motion/house/} is a widely used benchmark dataset consisting of 111 grayscale images of a toy house taken from different viewpoints. The graphs are constructed as described in the previous paragraph, but without node attributes, to evaluate the algorithms from a different perspective. Matching pairs consist of the first image and subsequent images with 5 image sequence gaps (such as image 1 vs. image 6 and so on).


\end{document}